\documentclass[journal]{IEEEtran}

\usepackage{cite}
\usepackage{amsmath}
\usepackage{amsthm}
\usepackage{amssymb}
\usepackage{graphicx}
\usepackage{epstopdf}
\usepackage[skip=2pt,font=footnotesize]{caption}
\usepackage{float}
\usepackage{algorithm}
\usepackage[noend]{algpseudocode}
\usepackage{multirow}
\usepackage[usenames, dvipsnames]{color}

\allowdisplaybreaks

\newtheorem{lemma}{Lemma}
\newtheorem{prop}{Proposition}

\newcommand{\bR}{\textbf{R}}

\newcommand{\Tr}{\text{Tr}}
\newcommand{\diag}{\text{diag}}
\newcommand{\Diag}{\text{Diag}}
\newcommand{\vect}{\text{vec}}

\newcommand{\rank}{\text{rank}}
\newcommand{\card}{\text{card}}
\newcommand{\epi}{\text{epi}}

\newcommand{\cN}{\mathcal{N}}
\newcommand{\cC}{\mathcal{C}}
\newcommand{\cA}{\mathcal{A}}
\newcommand{\cB}{\mathcal{B}}

\renewcommand{\vec}[1]{\boldsymbol{#1}}

\newcommand{\vx}{\vec{x}}
\newcommand{\vX}{\vec{X}}

\newcommand{\vY}{\vec{Y}}
\newcommand{\vw}{\vec{w}}
\newcommand{\vW}{\vec{W}}
\newcommand{\vz}{\vec{z}}

\newcommand{\va}{\vec{a}}

\newcommand{\vS}{\vec{S}}

\newcommand{\vA}{\vec{A}}

\newcommand{\vD}{\vec{D}}

\newcommand{\vR}{\vec{R}}
\newcommand{\vG}{\vec{G}}

\newcommand{\vH}{\vec{H}}

\newcommand{\vP}{\vec{P}}
\newcommand{\vQ}{\vec{Q}}

\newcommand{\vF}{\vec{F}}
\newcommand{\vI}{\vec{I}}
\newcommand{\vU}{\vec{U}}
\newcommand{\vu}{\vec{u}}
\newcommand{\vV}{\vec{V}}
\newcommand{\vv}{\vec{v}}
\newcommand{\vPsi}{\vec{\Psi}}

\newcommand{\vphi}{\vec{\phi}}
\newcommand{\vLambda}{\vec{\Lambda}}
\newcommand{\vlambda}{\vec{\lambda}}
\newcommand{\vXi}{\vec{\Xi}}
\newcommand{\vxi}{\vec{\xi}}
\newcommand{\vSigma}{\vec{\Sigma}}
\newcommand{\vzr}{\vec{0}}
\newcommand{\vone}{\vec{1}}

\begin{document}

\title{Orthogonal Sparse PCA and Covariance Estimation via Procrustes Reformulation}

\author{Konstantinos Benidis, Ying Sun, Prabhu Babu, 		and~Daniel~P.~Palomar,~\IEEEmembership{Fellow,~IEEE}%
\thanks{The authors are with the Department of Electronic and Computer Engineering, Hong Kong University of Science and Technology (HKUST), Hong Kong. E-mail: \{kbenidis, ysunac, eeprabhubabu, palomar\}@ust.hk}}

\maketitle

\begin{abstract}
	The problem of estimating sparse eigenvectors of a symmetric matrix attracts a lot of attention in many applications, especially those with high dimensional data set. While classical eigenvectors can be obtained as the solution of a maximization problem, existing approaches formulate this problem by adding a penalty term into the objective function that encourages a sparse solution. However, the resulting methods achieve sparsity at the expense of sacrificing the orthogonality property. In this paper, we develop a new method to estimate dominant sparse eigenvectors without trading off their orthogonality. The problem is highly non-convex and hard to handle. We apply the MM framework where we iteratively maximize a tight lower bound (surrogate function) of the objective function over the Stiefel manifold. The inner maximization problem turns out to be a rectangular Procrustes problem, which has a closed form solution. In addition, we propose a method to improve the covariance estimation problem when its underlying eigenvectors are known to be sparse. We use the eigenvalue decomposition of the covariance matrix to formulate an optimization problem where we impose sparsity on the corresponding eigenvectors. Numerical experiments show that the proposed eigenvector extraction algorithm matches or outperforms existing algorithms in terms of support recovery and explained variance, while the covariance estimation algorithms improve significantly the sample covariance estimator. 
\end{abstract}


\section{Introduction}
\label{sec:intro}

Principal Component Analysis (PCA) is a popular technique for data analysis and dimensionality reduction \cite{Jolliffe2002principal}. It has been used in various fields of engineering and science with a large number of applications such as machine learning, financial asset trading, face recognition, and gene expression data analysis. Given a data matrix $\vA\in\bR^{n\times m}$, with $\rank(\vA)=r$, PCA finds sequentially orthogonal unit vectors $\vv_1,\dots,\vv_r$, such that the variance of $\vA\vv_i$, which essentially is the projection of the data on the direction $\vv_i$, for $i=1,\dots,r$, is maximized. The directions $\vv_i$ are known as principal component (PC) loadings while $\vA\vv_i$ are the corresponding principal components (PCs). The PC loadings correspond to the right singular vectors of $\vA$ or to the eigenvectors of the corresponding sample covariance matrix $\vS=\frac{1}{n}\vA^T\vA$. 

PCA has many optimal properties that made it so widely used. First, it captures the directions of maximum variance of the data, thus we can compress the data with minimum information loss. Further, these directions are orthogonal to each other, i.e., they form an orthonormal basis. Finally, the PCs are uncorrelated which aids further statistical analysis. On the other hand, a particular disadvantage of PCA is that the PCs are usually linear combinations of all variables, i.e., the eigenvectors of $\vS$ are dense. Even if the underlying covariance matrix from which the samples are generated indeed has sparse eigenvectors, we do not expect to get a sparse result due to estimation error. Further, in many applications, the PCs have an actual physical meaning (e.g. gene expression). Thus, a sparse eigenvector could help significantly the interpretability of the result. 

Many different techniques have been proposed in this direction during the last two decades. In one of the first approaches, Jolliffe used various rotating techniques to obtain sparse loading vectors \cite{Jolliffe95rotation}. He showed though that it is impossible to preserve both the orthogonality of the loadings and the uncorrelatedness of the rotated components. In the same year, Cadima and Jolliffe suggested to simply set to zero all the elements that their absolute value is smaller than a threshold \cite{Cadima95loading}. In \cite{Jolliffe03modified}, the authors propose the SCoTLASS algorithm which maximizes the Rayleigh quotient of the covariance matrix, while sparsity is enforced with the Lasso penalty \cite{Tibshirani96lasso}. Many recent approaches are based on reformulations or convex relaxations. For example in \cite{Zou06SparsePCA}, Zou et al. formulate the sparse PCA problem as a ridge regression problem and they impose sparsity again using the Lasso penalty. In \cite{d'Aspremont07direct}, d'Aspremont et al. form a semidefinite program (SDP) after a convex relaxation of the sparse PCA problem, leading to the DSPCA algorithm. In \cite{d'Aspremont08Optimal}, the authors propose a greedy algorithm accompanied with a certificate of optimality. Low rank approximation of the data matrix is considered in \cite{Shen08Sparse}, under sparsity penalties, while in \cite{Journee10GeneralizedSPCA}, Journe\'e et al. reformulated the problem as an alternating optimization problem, resulting in the GPower algorithm. This algorithm turns out to be identical to the rSVD algorithm in \cite{Shen08Sparse}, except for the initialization and the post-processing phases. Similar power-type truncation methods were considered in \cite{Witten09penalized,Ma13Sparse}. In \cite{Yuan13truncated}, the authors ppropose a truncated power iteration
method. This method is similar to the classical power method, with
an additional truncation operation to ensure sparsity. Finally, in \cite{Song15Sparse}, the  sparse generalized eigenvalue problem is considered only for the first principal component, where the minorization-maximization (MM) framework is used. 

In all the aforementioned algorithms, apart from the fact that the PCs are correlated, the orthogonality property of the loadings is also sacrificed for sparse solutions. The only exception is the SCoTLASS algorithm that is suboptimal in the sense that it does not find jointly a sparse basis, but sequentially. The advantages of an orthogonal basis are well known. For instance, an orthonormal basis can be extremely useful since it can reduce the potential computational cost of any post-processing phase; this may not seem much for vector spaces of small dimension but it is invaluable for high dimensional vector spaces or function spaces. Consider for example the solution of a linear system via Gaussian elimination. It requires $O(m^3)$ operations for a non-orthogonal basis, compared to $O(m)$ operations if the basis is orthogonal, where $m$ is the dimension. This, among other optimal properties, motivates us to find sparse loading vectors that maintain their orthogonality.

Another issue in many contemporary applications is that the number of features $m$ in the corresponding datasets is extremely large while in many cases the number of samples $n$ is limited. It is well know by now that the sample covariance $\vS$ can be a very poor estimate of the population covariance matrix $\vSigma$ if the number of samples is restricted. Since the population covariance matrix $\vSigma$ is unknown, the classical PCA estimates the leading population eigenvectors by the sample covariance matrix $\vS$, which coincides with the maximum likelihood estimator (MLE) if $n\geq m$ and under the assumption that the samples are independent and identically distributed (i.i.d.), drawn from an $m$-dimensional Gaussian distribution. Many methods have been proposed to improve the covariance estimation in different settings and for different applications, e.g., for some representative works see \cite{Tadjudin98Covariance,Ledoit04honey,Ledoit04Well,Bickel08Regularized,Sun15Regularized,Friedman08Sparse,Levina08Sparse,d'Aspremont08First} and references therein. None of them has considered though to combine the prior information of sparsity in the eigenvectors with the covariance estimation.  

In this paper we focus and solve the two aforementioned problems: 1) the orthogonal sparse eigenvector extraction and 2) the joint covariance estimation with sparse eigenvectors. First, we apply the MM framework on the sparse PCA problem which results in solving a sequence of rectangular Procrustes problems. With this approach, we obtain sparse results but with the orthogonality property retained.
Then, we consider low sample settings where the population covariance matrices are known to have sparse eigenvectors. We formulate a covariance estimation problem where we impose sparsity on the eigenvectors. We propose two methods, i.e., alternating and joint estimation of the eigenvalues and eigenvectors, based on the MM framework. Both methods reduce to an iterative closed-form update with bounded iterations for the eigenvalues and a sequence of Procrustes problems for the eigenvectors, which maintain their orthogonality. 

Throughout the paper we consider real-valued matrices for simplicity. However, all the results hold for complex-valued matrices with trivial modifications: in the complex-valued case $|x_i|$ denotes the modulus of $x_i$ rather than the absolute value, while we should replace the transpose operation (i.e., $(\cdot)^T$) with the conjugate transpose operation (i.e., $(\cdot)^H$). Finally, we do not assume direct access to the data matrix $\vA$. Nevertheless, all the formulations hold if either the data matrix $\vA$ or the sample covariance matrix $\vS$ is provided.

The rest of the paper is organized as follows: In Section \ref{sec:prob_st} we first formulate the sparse eigenvector extraction and the covariance estimation problems. Then, we give a short review of the MM framework which will be the main tool to tackle both of the aforementioned problems. Finally we present the Procrustes problem since the solution of both of our problems involve certain Procrustes reformulations. In Section \ref{sec:sp_eig_est} we present the solution of the sparse eigenvector extraction problem. In Section \ref{sec:cov_est} we consider the problem of joint covariance estimation with sparse eigenvectors and we propose two algorithms to iteratively minimize the associated objective function. Section \ref{sec:num_exper} presents numerical experiments on artificial and real data and the conclusions are given in Section \ref{sec:conclusion}. 

\textit{Notation:} $\bR$ denotes the real field, $\bR^m$ ($\bR^m_+$) the set of (non-negative) real vectors of size $m$, and $\bR^{n\times m}$ the set of real matrices of size $n\times m$. Vectors are denoted by bold lower case letters and matrices by bold capital letters i.e., $\vx$ and $\vX$, respectively. The $i$-th entry of a vector is denoted by $x_i$, the $i$-th column of matrix $\vX$ by $\vx_{i}$, and the ($i$-th,$j$-th) element of a matrix by $x_{ij}$. A size $m$ vector of ones is denoted by $\vone_m$, while $\vI_m$ denotes the identity matrix of size $m$. $\vect(\cdot)$ denotes the vectorized form of a matrix. The superscripts $(\cdot)^T$ and $(\cdot)^H$ denote the transpose and conjugate transpose of a matrix, respectively, and $\Tr(\cdot)$ its trace. $\Diag(\vX)$ is a column vector consisting of all the diagonal elements of $\vX$ and $\diag(\vx)$ is a diagonal matrix formed with $\vx$ at its principal diagonal. Given a vector $\vx\in\bR^{mn}$, $[\vx]_{m\times n}$ is an $m\times n$ matrix such that $\vect([\vx]_{m\times n})=\vx$. $\|\vx\|_0$ denotes the number of nonzero elements of a vector $\vx\in\bR^m$. $\vS\succcurlyeq0$ means that the symmetric matrix $\vS$ is positive semidefinite, while $\lambda_{\text{max}}^{(\vS)}$ denotes its maximum eigenvalue. $\vX\otimes\vY$ is the Kronecker product of the matrices $\vX$ and $\vY$. $\cN(\vec{\mu},\vSigma)$ denotes the normal distribution with mean $\vec{\mu}$ and covariance matrix $\vSigma$. $\card(\cA)$ denotes the cardinality of the set $\cA$, $\cA\bigcup\cB$ denotes the union of the sets $\cA$ and $\cB$, and $\cA\setminus\cB$ their difference. $[i:j]$ with $i\leq j$, denotes the set of all integers between (and including) $i$ and $j$.

\section{Problem Statement and Background}
\label{sec:prob_st}

\subsection{Sparse Eigenvector Extraction}

Given a data matrix $\vA\in\bR^{n\times m}$, encoding $n$ samples of dimension $m$, we can extract the leading eigenvector of the scaled sample covariance matrix $\vS=\vA^T\vA$ by solving the following optimization problem:
\begin{equation}
	\begin{aligned}
		\underset{\vu}{\text{maximize}}&\quad \vu^T\vS\vu\\
		\text{subject to}&\quad \vu^T\vu=1.
	\end{aligned}
\end{equation}
In order to  get a sparse result, we can include a regularization term in the objective that imposes sparsity, i.e., 
\begin{equation}
	\begin{aligned}\label{opt:single_eigenvec estimation}
		\underset{\vu}{\text{maximize}}&\quad \vu^T\vS\vu - \rho\|\vu\|_0\\
		\text{subject to}&\quad \vu^T\vu=1,
	\end{aligned}
\end{equation}
where $\rho$ is a regularization parameter.

Problem \eqref{opt:single_eigenvec estimation} can be generalized to extract multiple eigenvectors as follows:
\begin{equation}
	\begin{aligned}\label{gen_sparse_eig_opt}
		\underset{\vU}{\text{maximize}}&\quad \Tr\left(\vU^T\vS\vU\vD\right) - \sum_{i=1}^{q}\rho_i\|\vu_i\|_0\\
		\text{subject to}&\quad \vU^T\vU=\vI_q.
	\end{aligned}
\end{equation}
Here, $q$ is the number of eigenvectors we wish to estimate, $\vU\in\bR^{m\times q}$, and $\vD\succcurlyeq0$ is a diagonal matrix giving weights to the different eigenvectors. In the case where $q=m$, $\vD$ should be different from the (scaled) identity matrix since the first term reduces to a constant and $\vU^\star=\vP_m$, where $\vP_m$ is a permutation matrix of size $m$. 

The optimization problem \eqref{gen_sparse_eig_opt} involves the maximization of a non-concave discontinuous objective function over a non-convex set, thus the problem is too hard to deal with directly.

In order to deal with the discontinuity of the $\ell_0$-norm, we approximate it by a continuous function $g_p(x)$, where $p>0$ is a parameter that controls the approximation. Following \cite{Song15Sparse}, we consider an even function defined on $\bR$, which is differentiable everywhere except at $0$, concave and monotone increasing on $[0,+\infty)$, with $g_p(0)=0$. Among the functions that satisfy the aforementioned criteria, in this paper we choose the function 
\begin{equation}
	g_p(x)=\frac{\log\left(1+|x|/p\right)}{\log\left(1+1/p\right)},
\end{equation}
with $0<p\leq1$. This function is also used to replace the $\ell_1$-norm in \cite{Candes08Enhancing}, and leads to the iteratively reweighted $\ell_1$-norm minimization algorithm.

The function $g_p(\cdot)$ is not smooth which may cause an optimization algorithm to get stuck at a non-differentiable point \cite{Figueiredo07Majorization}. To handle non-smoothness of $g_p(\cdot)$
we use a smoothened version, based on Nesterov's smooth minimization technique presented in \cite{Nesterov05Smooth} and following the results of \cite{Song15Sparse}, which is defined as:
\begin{equation}\label{eq:L0 approximaion function}
	g_p^{\epsilon}\left(x \right)= \begin{cases}
		\frac{x^2}{2\epsilon(p+\epsilon)\log(1+1/p)},& |x|\leq\epsilon,\\
		\frac{\log\left(\frac{p+|x|}{p+\epsilon}\right)+\frac{\epsilon}{2(p+\epsilon)}}{\log(1+1/p)},& |x|>\epsilon,
	\end{cases}
\end{equation}
with $0<p\leq1$ and $0<\epsilon\ll1$.

This leads to the following approximate problem:
\begin{equation}
	\begin{aligned} 
		\underset{\vU}{\text{maximize}}&\quad \textrm{Tr}\left(\vU^T\vS\vU\vD\right) - \sum_{j=1}^{q}\rho_j \sum_{i=1}^{m}g_p^{\epsilon}\left(u_{ij} \right)\\
		\text{subject to}&\quad \vU^T\vU=\vI_q.
	\end{aligned}\label{eq:gen_sparse_eig_opt_approximate}
\end{equation}

The problem presented in \cite{Song15Sparse}, is a special case of the above optimization problem, with $q=1$. Nevertheless, it is not possible to follow the same procedure as in \cite{Song15Sparse} to solve the problem due to the orthogonality constraint. Instead, we tackle this problem using the MM algorithm, which results in solving a sequence of rectangular Procrustes problems that have a closed-form solution based on singular value decomposition (SVD). 

\subsection{Covariance Estimation}

We first consider a typical covariance estimation problem. We assume that the random variable $\vx\in\bR^m$ follows a zero mean Gaussian distribution with covariance $\vec{\Sigma}$, i.e.,  $\vx\sim\cN(\vzr,\vec{\Sigma})$. Given $n\geq m$ i.i.d. samples $\vx_i$, with $i=1,\dots,n$, our goal is to estimate $\vec{\Sigma}$. The maximum likelihood estimator of $\vec{\Sigma}$ is given by the solution of the following problem:
\begin{equation}
	\begin{aligned}\label{opt:covariance estimation}
		\underset{\vec{\Sigma}}{\text{minimize}}&\quad\log\det\left(\vec{\Sigma}\right)+\Tr\left(\vS\vec{\Sigma}^{-1}\right)\\
		\text{subject to}&\quad \vec{\Sigma}\succcurlyeq0,
	\end{aligned}
\end{equation}
where $\vS$ is the sample covariance matrix, i.e.,
\begin{equation}
	\vS=\frac{1}{n}\sum_{i=1}^{n}\vx_i\vx_i^T.
\end{equation}
The above problem is not convex but it can be easily transformed into a convex one by setting $\vPsi=\vec{\Sigma}^{-1}$. With this transformation we get:
\begin{equation}
	\begin{aligned}\label{opt:covariance estimation inverse}
		\underset{\vPsi}{\text{minimize}}&\quad-\log\det\left(\vPsi\right)+\Tr\left(\vS\vPsi\right)\\
		\text{subject to}&\quad \vPsi\succcurlyeq0.
	\end{aligned}
\end{equation}
The optimal solution of this problem is $\vPsi=\vS^{-1}$, thus, the MLE of the covariance matrix is $\vec{\Sigma}=\vS$, which is simply the sample covariance matrix.

Now, we would like to estimate the population covariance matrix $\vSigma$ while we impose sparsity on its eigenvectors. Thus, we need to reformulate the covariance estimation problem in terms of eigenvalues and eigenvectors. Further, we add a cardinality penalty on the $q$ principal eigenvector. Notice though that we estimate all $m$ eigenvectors and not only the $q$ principal ones since it is a covariance estimation and not an eigenvector extraction problem. 

Consider the eigenvalue decomposition of $\vPsi$, i.e., $\vPsi=\vU\vLambda\vU^T$, with $\vU,\vLambda\in\bR^{m\times m}$ and $\vLambda=\diag(\vlambda)\succcurlyeq0$. Then, we can formulate our problem as follows:
\begin{align}\label{opt:cov_est_block_MM}
	\underset{\vU,\vLambda}{\text{minimize}}&\!\quad-\!\log\det\!\left(\!\vLambda\!\right)\!+\!\Tr\left(\vS\vU\vLambda\vU^T\right)\!\!+\!\sum_{i=1}^{q}\rho_i\|\vu_i\|_0\nonumber\\
	\text{subject to}&\quad \vLambda\succcurlyeq0,\nonumber\\
	&\quad \lambda_{i}\leq \lambda_{i+1}, \quad i=1,\dots, q-1,\\
	&\quad \lambda_{q}\leq \lambda_{q+i}, \quad i=1,\dots,m-q,\nonumber\\
	&\quad \vU^T\vU=\vI.\nonumber
\end{align}
Let us first make some comments on the above problem. We penalize the cardinality of the first $q\leq m$ principal eigenvectors where each of them is associated with a different sparsity inducing parameter $\rho_i$. Thus, we need to keep the order of the first $q$ eigenvectors intact. We succeed this by imposing ordering to the corresponding eigenvalues. Notice also that the principal eigenvector corresponds to the smallest eigenvalue of $\vPsi$ since  $\vPsi=\vSigma^{-1}$. 

It will be useful in the following to expand the sparsity term and include all eigenvectors by setting the redundant sparsity inducing parameters to zero, i.e., $\rho_i=0$ for $i=q+1,\dots,m$. Again, we approximate the $\ell_0$-norm by a differentiable function $g_p^{\epsilon}\left(\cdot \right)$, given by \eqref{eq:L0 approximaion function}. This leads to the following approximate problem:
\begin{equation}
	\begin{aligned}\label{opt:cov_est_block_MM_approxL0}
		\underset{\vU,\vLambda}{\text{minimize}}&\quad-\log\det\left(\vLambda\right)+\Tr\left(\vS\vU\vLambda\vU^T\right)\\
		&\quad +\sum_{j=1}^{m}\rho_j \sum_{i=1}^{m}g_p^{\epsilon}\left(u_{ij}\right)\\
		\text{subject to}&\quad \vLambda\succcurlyeq0,\\
		&\quad \lambda_{i}\leq \lambda_{i+1}, \quad i=1,\dots, q-1,\\
		&\quad \lambda_{q}\leq \lambda_{q+i}, \quad i=1,\dots,m-q,\\
		&\quad \vU^T\vU=\vI.
	\end{aligned}
\end{equation}

Although in \eqref{opt:cov_est_block_MM_approxL0} we have approximated the objective of \eqref{opt:cov_est_block_MM} with a continuous and differentiable function, the problem still remains too hard to deal with directly since it involves the minimization of a non-convex function over a non-convex set. 

\subsection{Shrinkage}

In the case where the number of samples is less than the dimension of the problem, i.e., when $n<m$, the sample covariance matrix $\vS$ is low rank. As a result, all the covariance estimation problems that were presented are unbounded below.  

We can overcome this problem, for example, by shrinking the sample covariance matrix towards an identity matrix \cite{Abramovich1982Controlled,Ledoit04honey}, i.e.,
\begin{equation}\label{eq:S_sh}
	\vS_{\text{sh}}=(1-\delta)\vS+\delta\vI_m,
\end{equation}
with $0<\delta\leq1$. With this technique we bound the minimum eigenvalue of $\vS_{\text{sh}}$ by $\delta$, the matrix becomes full rank and the optimization problems are now well defined. The effect of shrinkage in the estimation of $\vSigma$ will be shown in Section \ref{sec:num_exper}. 

\subsection{Review of the MM framework}

The minorization-maximization (if we maximize) or majorization-minimization (if we minimize) algorithm is a way to handle optimization problems that are too difficult to face directly \cite{Hunter04TutorialMM}. Consider a general optimization problem
\[
\begin{aligned}
\underset{\vx}{\text{maximize}}& \quad f\left(\vx\right)\\
\text{subject to}&\quad \vx \in \mathcal{X},
\end{aligned}
\label{eq:P}
\]
where $\mathcal{X}$ is a closed set. At a given point $\vx^{(k)}$, the minorization-maximization algorithm finds a surrogate function $g\left( \vx |\vx^{(k)}\right)$ of $f\left(\vx \right)$ satisfying the following properties: 
\begin{itemize}
	\item $f\left(\vx^{(k)} \right) = g\left( \vx^{(k)} |\vx^{(k)}\right)$,\\
	\item $f\left(\vx\right) \geq g\left( \vx |\vx^{(k)}\right),\, \forall \vx \in \mathcal{X}$.
\end{itemize}
Then $\vx$ is iteratively updated (with $k$ denoting iterations) as:
\begin{equation}
	\vx^{(k+1)} = \arg\max_{\vx \in \mathcal{X}} g\left( \vx |\vx^{(k)}\right).
\end{equation}
It can be seen easily that $f\left(\vx^{(k)}\right)\leq f\left(\vx^{(k+1)}\right)$ holds. 

The majorization-minimization algorithm works in an equivalent way, such that in each update  $f\left(\vx^{(k)}\right)\geq f\left(\vx^{(k+1)}\right)$ holds.

In practice, it is not a trivial task to find a surrogate function such that the maximizer of the minorization (or minimizer of the majorization) function of the objective can be found easily or even have a closed-form solution. The following lemma will be useful for the MM algorithms that will be derived throughout this paper:
\begin{lemma}\label{lemma:zero norm upperbound}
	On the set $\big\{\vU\in\bR^{m\times q}\vert\vU^T\vU=\vI_q\big\}$, the function $\sum_{j=1}^{q}\rho_j \sum_{i=1}^{m}g_p^{\epsilon}\left(u_{ij}\right)$ is majorized at $\vU_0$ by $2\Tr\left({\vH}^T\vU\right) + c$, where
	\begin{equation}
		\vH=\left[\diag\left(\vw-\vw_{\max}\otimes\vone_{m}\right)\vect(\vU_0)\right]_{m\times q}
	\end{equation} 
	and 
	\begin{equation}\label{eq:const c}
		c=2\left(\vone_q^T\vw_{\max}\right)-{\vect(\vU_0)}^T\diag\left(\vw\right)\vect(\vU_0).
	\end{equation}
	The weights $\vw\in\bR_+^{mq}$ are given by 
	\begin{equation}\label{eq:weights}
		w_{i}= \begin{cases}
			\frac{\rho_i}{2\epsilon(p+\epsilon)\log(1+1/p)},& |u_{0,i}|\leq\epsilon,\\
			\frac{\rho_i}{2\log(1+1/p)|u_{0,i}|\left(|u_{0,i}|+p\right)},& |u_{0,i}|>\epsilon,
		\end{cases}
	\end{equation}
	where $\vu_0=\vect\left(\vU_0\right)$, and $\vw_{\max}\in\bR^q_+$, with $w_{\max,i}$ being the maximum weight that corresponds to $\vu_{0,i}$.
\end{lemma}  
\begin{proof}
	See Appendix \ref{app:proof of zero norm upperbound}.
\end{proof}

\subsection{Procrustes problems}

Consider the following optimization problem:
\begin{equation}
	\begin{aligned}\label{opt:procrustes_trace form}
		\underset{\vX}{\text{maximize}}&\quad \Tr\left(\vY^T\vX\right)\\
		\text{subject to}&\quad \vX^T\vX=\vI_q,
	\end{aligned}
\end{equation}
where $\vX,\vY\in\bR^{m\times q}$. 
Notice that problem \eqref{opt:procrustes_trace form} is equivalent to 
\begin{equation}\label{Procrustes}
	\begin{aligned}
		\underset{\vX}{\text{minimize}}&\quad \|\vX - \vY\|_F^2\\
		\text{subject to}&\quad \vX^T\vX=\vI_q,
	\end{aligned}
\end{equation}
which is a Procrustes problem.

\begin{lemma}\label{lemma:Procrustes}
	For $m=q$ ($m>q$), problem \eqref{opt:procrustes_trace form} can be transformed into an orthogonal (rectangular) Procrustes problem and its optimal solution is $\vX^\star=\vV_L\vV_R^T$, where $\vV_L,\vV_R$ are the left and right singular vectors of the matrix $\vY$, respectively \cite{Schonemann66Generalized},\cite[Proposition 7]{Manton02Optimization}.
\end{lemma}

\section{Sparse PCA}
\label{sec:sp_eig_est}

In this section we return to the sparse eigenvector extraction problem as formulated in \eqref{eq:gen_sparse_eig_opt_approximate}. In the following, we apply the MM algorithm and derive a tight lower bound (surrogate function), $g\left(\vU|\vU^{(k)}\right)$, for the objective function of \eqref{eq:gen_sparse_eig_opt_approximate}, denoted by $f\left(\vU\right)$, at the $\left(k+1\right)$-th iteration.

\begin{prop} The function $f\left(\vU\right)$ is lowerbounded by the surrogate function 
	\begin{equation}
		g\left(\vU|\vU^{(k)}\right) = 2\Tr\left(\!\left(\!\vG^{(k)}-\vH^{(k)}\!\right)^T\vU\!\right)+c_1-c_2,
	\end{equation}
	where 
	\begin{equation}\label{eq:G}
		\vG^{(k)} =\vS\vU^{(k)}\vD,
	\end{equation}
	\begin{equation}\label{eq:H}
		\vH^{(k)}\!=\!\left[\diag\left(\vw^{(k)}-\vw_{\max}^{(k)}\otimes\vone_{m}\right)\vu^{(k)}\right]_{m\times q},
	\end{equation}
	and $c_1,c_2$ are optimization irrelevant constants. Equality is achieved when $\vU=\vU^{(k)}$.
\end{prop}

\begin{proof}
	The first term of the objective is convex so a lower bound can be constructed by its first order Taylor expansion:
	\begin{align}
		\Tr\left(\vU^T\vS\vU\vD\right)\geq 2\Tr\left(\left(\vS\vU^{(k)}\vD\right)^T\vU\right) + c_1,
	\end{align}
	where $c_1 = -\Tr\left({\vU^{(k)}}^T\vS\vU^{(k)}\vD\right)$
	is a constant. 
	
	For the second term, using the results from Lemma \ref{lemma:zero norm upperbound}, it is straightforward to show that it is lowerbounded by the function $-2\Tr\left(\vH^{(k)}\vU\right)-c_2$, where $\vH^{(k)}$ is given by \eqref{eq:H} and $c_2=2\left(\vone_q^T\vw_{\max}\right)-{\vu^{(k)}}^T\diag\left(\vw\right)\vu^{(k)}$
	is a constant.
\end{proof}

Now, we drop the constants and the optimization problem of every MM iteration takes the following form:
\begin{equation}
	\begin{aligned}\label{opt:sp_eig_trace form}
		\underset{\vU}{\text{maximize}}&\quad \Tr\left(\left(\vG^{(k)}-\vH^{(k)}\right)^T\vU\right)\\
		\text{subject to}&\quad \vU^T\vU=\vI_q.
	\end{aligned}
\end{equation}

\begin{prop}
	The optimal solution of the optimization problem \eqref{opt:sp_eig_trace form} is $\vU^\star=\vV_L\vV_R^T$, where $\vV_L\in\bR^{m\times q}$ and $\vV_R\in\bR^{q\times q}$ are the left and right singular vectors of the matrix $\left(\vG^{(k)}-\vH^{(k)}\right)$, respectively.
\end{prop}
\begin{proof}
	The proof comes directly from Lemma \ref{lemma:Procrustes}.
\end{proof}

In Algorithm \ref{alg:Alg1} we summarize the above iterative procedure. We will refer to it as IMRP.

Since the algorithm does not perform any hard thresholding, the resulting eigenvectors do not have zero elements but rather very small values. To this end, we can set to zero all the values that are below a threshold (e.g. $10^{-12}$) and obtain sparse eigenvectors. As it will be shown in the numerical experiments, the affect of this thresholding on the orthogonality of the eigenvectors is negligible.

\begin{algorithm}[t]
	\caption{IMRP - Iterative Minimization of Rectangular Procrustes for the sparse eigenvector problem \eqref{eq:gen_sparse_eig_opt_approximate}}\label{alg:Alg1}
	\begin{algorithmic}[1]
		\State Set $k=0$, choose $\vU^{(0)}\in\{\vU\vert\vU^T\vU=\vI_q\}$
		\State \textbf{repeat}:	
		\State\hspace{\algorithmicindent} Compute $\vG^{(k)},\vH^{(k)}$ with \eqref{eq:G}-\eqref{eq:H} 
		\State\hspace{\algorithmicindent} Compute $\vV_L$, $\vV_R$, the left and right singular vectors
		\Statex\hspace{\algorithmicindent} of $\left(\vG^{(k)}-\vH^{(k)}\right)$, respectively
		\State\hspace{\algorithmicindent} $\vU^{(k+1)} = \vV_L\vV_R^T$
		\State\hspace{\algorithmicindent} $k \gets k+1$
		\State \textbf{until} convergence
		\State \textbf{return} $\vU^{(k)}$
	\end{algorithmic}
\end{algorithm}

\subsection{Explained Variance}

In the ordinary PCA the principal components are uncorrelated while the corresponding loadings are orthogonal. If we denote by $\vY$ the ordinary principal components, the total explained variance can be calculated as $\Tr\left(\vY^T\vY\right)$. If the principal components are correlated though, computing the total variance this way will overestimate the true explained variance. 

An approach to overcome this issue was first suggested in \cite{Zou06SparsePCA} (and adopted in \cite{Journee10GeneralizedSPCA}), where the authors introduced the notion of adjusted variance. The idea is to remove the correlations of the principal components sequentially. This can be done efficiently by the QR decomposition: if $\vA\in\bR^{n\times m}$ is a data matrix and $\vU\in\bR^{m\times q}$ are the $q$ estimated loadings, then the adjusted variance is simply
\begin{equation}
	\text{AdjVar}\left(\vU\right)=\Tr\left(\vR^2\right),
\end{equation}
where $\vA\vU=\vQ\vR$, is the QR decomposition of $\vA\vU$. The explained variance percentage can be then computed as AdjVar$(\vU)/$AdjVar$(\vU_{\text{PCA}})$, where $\vU_{\text{PCA}}$ are the first $q$ eigenvectors of $\vA^T\vA$. 

As mentioned in \cite{Shen08Sparse}, in the above approach the lack of orthogonality in the loadings is not addressed. Thus, a new approach was proposed: when the loading vectors are not orthogonal we should not consider separate projections of the data matrix onto each of them. Instead, we should project the data matrix onto the $q$-dimensional subspace, i.e., $\vA_q=\vA\vU\left(\vU^T\vU\right)^{-1}\vU^T$. Then, the total variance is simply $\Tr\left(\vA_q^T\vA_q\right)$ and the cumulative percentage of explained variance (CPEV) can be computed as
\begin{equation}\label{eq:CPEV}
	\text{CPEV}=\Tr\left(\vA_q^T\vA_q\right)/\Tr\left(\vA^T\vA\right).
\end{equation}

In this paper we adopt the second approach and compute the explained variance using \eqref{eq:CPEV}. 

\begin{table*}[t!]
	\centering
	\begin{tabular}{|l|c|c|c|c|}
		\hline
		& Case 1 & \multicolumn{2}{c|}{Case 2} & Case 3  \\ 
		\hline
		\multirow{4}{*}{Conditions} &  & \multicolumn{2}{c|}{$z_{j-1}^{(k)}> z_{j}^{(k)}\quad\text{if }j>1$ } & $z_{q-r-1}^{(k)}> z_{q-r}^{(k)}$\\
		& $z_i^{(k)}\!\!\geq\!\! z_{i+1}^{(k)},\!\!\!\!\! \quad i\!\in\![1\!:\!q\!-\!1]$ & \multicolumn{2}{c|}{$z_i^{(k)}\leq z_{i+1}^{(k)}, \quad i\in[j:j+k-1]$}  & $z_{i}^{(k)}\leq z_{i+1}^{(k)}, \quad i\in[q-r:q-1]$  \\ \cline{3-4}
		& $z_q^{(k)}\!\!\geq\!\! z_{q+i}^{(k)},\!\!\!\!\! \quad i\!\in\![1\!:\!m\!-\!q]$ & $\text{if }j+k<q$ &$\text{if }j+k=q$  & $z_{q}^{(k)}\leq z_{c_i}^{(k)}, \quad i\in[1:k]$\\ 
		&    & $z_{j+k}^{(k)}\!>\!z_{j+k+1}^{(k)}$ & $\!z_q^{(k)}\!\!>\!\! z_{q+i}^{(k)},\!\!\!\!\! \quad i\!\in\![1\!:\!m\!-\!q]\!$  &  $z_{q}^{(k)}> z_{q+i}^{(k)}, \quad i\in[1:m-q]\setminus\cC$\\
		\hline
		\multirow{2}{*}{$\!$Block Updates$\!\!$}     & \multirow{2}{*}{-} & \multicolumn{2}{c|}{\multirow{2}{*}{$z_i^{(k+1)}=\frac{1}{k+1}\sum_{i=0}^{k}\limits z_{j+i}^{(k)},\!\!\! \quad i\in[j:j+k]$}} & $z_i^{(k+1)}=z_i^{(k)},\quad i\!\in\!\cC\setminus\cA$\\
		& & \multicolumn{2}{c|}{} & $\!z_i^{(k+1)}\!\!=\!\!\frac{1}{r+p+1}\!\sum_{i=0}^{r}\limits\!z_{q-i}^{(k)}\!+\!\sum_{i=1}^{p}\limits\!z_{a_i}^{(k)},\!\!\quad i\!\in\![q\!-\!r\!:\!q]\bigcup\cA\!$\\
		\hline
		Solution  & $\vlambda^\star=\frac{1}{\vz^{(k)}}$ & \multicolumn{2}{c|}{-} & -\\
		\hline
	\end{tabular}
	\caption{Updates and optimal solution of the iterative procedure that solves the optimization problem \eqref{opt:lambda_block_MM}.}
	\label{tab:lambda_block_MM solution}
\end{table*}

\section{Sparse Eigenvectors in Covariance Estimation}
\label{sec:cov_est}

In this section we return to the problem of covariance estimation with sparse eigenvectors. We consider the formulation \eqref{opt:cov_est_block_MM_approxL0}, i.e.,  

\begin{equation*}
	\begin{aligned}
		\underset{\vU,\vLambda}{\text{minimize}}&\quad-\log\det\left(\vLambda\right)+\Tr\left(\vS\vU\vLambda\vU^T\right)\\
		&\quad +\sum_{j=1}^{m}\rho_j \sum_{i=1}^{m}g_p^{\epsilon}\left(u_{ij}\right)\\
		\text{subject to}&\quad \vLambda\succcurlyeq0,\\
		&\quad \lambda_{i}\leq \lambda_{i+1}, \quad i=1,\dots, q-1,\\
		&\quad \lambda_{q}\leq \lambda_{q+i}, \quad i=1,\dots,m-q,\\
		&\quad \vU^T\vU=\vI.
	\end{aligned}
\end{equation*}

To deal with this problem, we propose two methods based on the MM framework. In Section \ref{subsec:alt_opt} we perform alternating optimization of the eigenvalues and eigenvectors while in Section \ref{subsec:joint_opt} we estimate them jointly.

\subsection{Alternating Optimization Using the MM Framework}
\label{subsec:alt_opt}

We begin with the optimization problem \eqref{opt:cov_est_block_MM_approxL0} which is highly non-convex. We tackle it by alternating optimization of $\vU$ and $\vLambda$.

For fixed $\vU$ the optimization problem over $\vlambda$ can be written in the following convex form:
\begin{equation}
	\begin{aligned}\label{opt:lambda_block_MM}
		\underset{\vec{\lambda}}{\text{minimize}}&\quad-\sum_{i=1}^{m}\log\lambda_i+\sum_{i=1}^{m}z_i\lambda_i\\
		\text{subject to}&\quad \lambda_{i}\leq \lambda_{i+1}, \quad i=1,\dots,q-1,\\
		&\quad \lambda_{q}\leq \lambda_{q+i}, \quad i=1,\dots,m-q,
	\end{aligned}
\end{equation}
where we have dropped the positive semidefinite constraint of $\vLambda$ since it is implicit form the $\log$ function, and $\vz=\Diag\left(\vU^T\vS\vU\right)\geq\vzr$, since $\vS\succcurlyeq\vzr$.

The optimization problem \eqref{opt:lambda_block_MM} does not have a closed-form solution. Nevertheless, we can find an iterative closed-form update of the parameter $\vz$ that will allow us to obtain the optimal solution for $\vlambda$.

We start from the corresponding unconstrained version of problem \eqref{opt:lambda_block_MM} whose solution is 
\begin{equation}
	\vlambda=\frac{1}{\vz^{(0)}}, 
\end{equation}
where $\vz^{(0)}=\vz$. If the solution is feasible then it is the optimal one. Else, we need to update $\vz$. In every iteration, all the non-overlapping blocks of $z_i$'s that satisfy certain conditions need to be updated in parallel. In the $k$-th iteration we distinguish three different cases:\\ 
\textit{Case 1:} $\vlambda=1/\vz^{(k)}$ satisfies all the constraints of problem \eqref{opt:lambda_block_MM}. Then the optimal solution is $\vlambda^\star=\vlambda$.\\
\textit{Case 2:} $\vlambda=1/\vz^{(k)}$ violates $r\geq1$ consecutive ordering constraints of the first $q$ eigenvalues. For any such block violation we need to update $\vz^{(k)}$.\\  
\textit{Case 3:} $\vlambda=1/\vz^{(k)}$ violates $r+l\geq1$ consecutive ordering constraints, with $r\geq0$ and $l\geq1$, including the last $r+1$ ordered and a set of $l$ unordered eigenvalues. Since we do not impose ordering on the $m-q$ last eigenvalues, any of them could violate the inequality with $\lambda_q$ and not only the neighboring ones. Thus, we use the indices $c_1,\dots,c_l$, with $c_i>q$, for $i=1,\dots,l$, and $c_i\in\cC$, with $\cC$ the set of indices of the eigenvalues that violate the inequality constraints with $\lambda_q$. We further denote by $\cA\subseteq\cC$, with $\card(\cA)=p<l$, the set of indices given by
\begin{equation}
	\cA\!=\!\Bigg\{\!c_i\Bigg\vert z_{c_i}^{(k)}\!\geq\!\frac{1}{r\!+\!l\!-\!i\!+\!1}\!\left(\sum_{s=0}^{r}z_{q-s}^{(k)}\!+\!\!\sum_{s=0}^{l-i-1}\!\!z_{c_{l-s}}^{(k)}\!\right)\!\!\Bigg\}.
\end{equation}
For any such block violation we need to update $\vz^{(k)}$.

\begin{prop}\label{prop:lambda_block_MM}
	The iterative-closed form update procedure given in Table \ref{tab:lambda_block_MM solution} converges to the solution of problem \eqref{opt:lambda_block_MM}.
\end{prop}
\begin{proof}
	See Appendix \ref{app:proof of algorithm block MM}.
\end{proof}

Now, for fixed $\vLambda$ the problem over $\vU$ becomes:
\begin{equation}
	\begin{aligned}\label{opt:blockMM_optU}
		\underset{\vU}{\text{minimize}}&\quad \Tr\left(\vS\vU\vec{\Lambda}\vU^T\right)+\sum_{j=1}^{q}\rho_j \sum_{i=1}^{m}g_p^{\epsilon}\left(u_{ij}\right)\\
		\text{subject to}&\quad \vU^T\vU=\vI.
	\end{aligned}
\end{equation}

For the second term we can use the same bound as the one for problem \eqref{eq:gen_sparse_eig_opt_approximate}. However, we cannot linearize the first term as previously since the linear approximation is a lower and not an upper bound of a convex function. 

To minimize the objective function we apply the MM algorithm and derive a tight upper bound, $g_{\text{alt}}\left(\vU|\vU^{(k)}\right)$, for the objective function of \eqref{opt:blockMM_optU}, denoted by $f_{\text{alt}}\left(\vU\right)$, at the $\left(k+1\right)$-th iteration.

\begin{prop} The function $f_{\text{alt}}\left(\vU\right)$ is upper bounded by the surrogate function 
	\begin{equation}
		g_{alt}\!\left(\!\vU|\vU^{(k)}\!\right)\!=\! 2\Tr\left(\!\left(\!\vG_{\text{alt}}^{(k)}\!+\!\vH^{(k)}\!\right)^T\!\!\vU\!\right)+c_3+c_4,
	\end{equation}
	where 
	\begin{equation}\label{eq:G_alt}
		\vG_{\text{alt}}^{(k)}=\left[\left(\vLambda\otimes\left(\vS-\lambda_{\max}^{(\vS)}\vI_m\right)\right)\vu^{(k)}\right]_{m\times m},
	\end{equation} 
	\begin{equation}\label{eq:H_cov}
		\vH^{(k)}=\left[\diag\left(\vw^{(k)}-\vw_{\max}^{(k)}\otimes\vone_{m}\right)\vu^{(k)}\right]_{m\times m},
	\end{equation}
	and  $c_3,c_4$ are optimization irrelevant constants.
	Equality is achieved when $\vU=\vU^{(k)}$.
\end{prop}

\begin{proof}	
	For the first term of the objective it holds that
	\begin{equation}
		\Tr(\vS\vU\vec{\Lambda}\vU^T)=\vu^T(\vec{\Lambda}\otimes\vS)\vu,
	\end{equation}
	where $\vu=\vect(\vU)$. In a similar manner as in the proof of Lemma \ref{lemma:zero norm upperbound}, it is easy to show that the following holds: 
	\begin{equation}
		\vu^T(\vec{\Lambda}\otimes\vS)\vu\leq 2\Tr\left({\vG_{\text{alt}}^{(k)}}^T\vU\right) + c_3,
	\end{equation}  
	where
	$\vG_{\text{alt}}^{(k)}=\left[\left(\vLambda\otimes\left(\vS-\lambda_{\max}^{(\vS)}\vI_m\right)\right)\vu^{(k)}\right]_{m\times m}$
	and 
	$c_3=2\lambda_{\max}^{(\vS)}\vone^T\vlambda-{\vu^{(k)}}^T\!\left(\vLambda\otimes\vS\right)\vu^{(k)}$
	is a constant. 
	
	For the second term it is straightforward from Lemma \ref{lemma:zero norm upperbound} that an upper bound is the function $2\Tr\left({\vH^{(k)}}^T\vU\right)+c_4$, with
	$\vH^{(k)}=\left[\diag\left(\vw^{(k)}-\vw_{\max}^{(k)}\otimes\vone_{m}\right)\vu^{(k)}\right]_{m\times m}$ and
	$c_4\!=\!\vone_m^T\vw_{\max}^{(k)}-{\vu^{(k)}}^T\!\!\diag\!\left(\!\vw^{(k)}\!-\!\vw_{\max}^{(k)}\otimes\vone_{m}\!\right)\vu^{(k)}$
	a constant.
\end{proof}

Now, we drop the constants and the optimization problem of every MM iteration takes the following form:
\begin{equation}
	\begin{aligned}\label{opt:procrustes_trace form_blockMM}
		\underset{\vU}{\text{minimize}}&\quad \Tr\left(\left(\vG_{\text{alt}}^{(k)}+\vH^{(k)}\right)^T\vU\right)\\
		\text{subject to}&\quad \vU^T\vU=\vI_m.
	\end{aligned}
\end{equation}

\begin{prop}
	The optimal solution of the optimization problem \eqref{opt:procrustes_trace form_blockMM} is $\vU^\star=\vV_L\vV_R^T$, where $\vV_L\in\bR^{m\times m}$ and $\vV_R\in\bR^{m\times m}$ are the left and right singular vectors of the matrix $-\left(\vG_{\text{alt}}^{(k)}+\vH^{(k)}\right)$, respectively.
\end{prop}
\begin{proof}
	The proof comes directly from Lemma \ref{lemma:Procrustes}.
\end{proof}
In Algorithm \ref{alg:Alg2} we summarize the above iterative procedure. We will refer to it as AOCE.

\begin{algorithm}[t]
	\caption{AOCE - Alternating Optimization for Covariance Estimation for the problem \eqref{opt:cov_est_block_MM_approxL0}}\label{alg:Alg2}
	\begin{algorithmic}[1]
		\State Set $k=0$, choose $\vU^{(0)}\in\{\vU\vert\vU^T\vU=\vI_q\}$
		\State \textbf{repeat}:
		\State\hspace{\algorithmicindent} Compute $\vlambda^{(k+1)}$ from Proposition \ref{prop:lambda_block_MM}		
		\State\hspace{\algorithmicindent} Compute $\vG_{\text{alt}}^{(k)},\vH^{(k)}$ with \eqref{eq:G_alt}-\eqref{eq:H_cov}
		\State\hspace{\algorithmicindent} Compute $\vV_L$, $\vV_R$, the left and right singular vectors 
		\Statex\hspace{\algorithmicindent} of $-\left(\vG_{\text{alt}}^{(k)}+\vH^{(k)}\right)$, respectively
		\State\hspace{\algorithmicindent} $\vU^{(k+1)} = \vV_L\vV_R^T$
		\State\hspace{\algorithmicindent} $k \gets k+1$
		\State \textbf{until} convergence
		\State \textbf{return} $\vU^{(k)},\vlambda^{(k)}$
	\end{algorithmic}
\end{algorithm}

\subsection{Joint Optimization Using the MM Framework}
\label{subsec:joint_opt}

Let us consider again the formulation \eqref{opt:cov_est_block_MM_approxL0} with the variable transformation $\vXi=\vLambda^{-1}$. The optimization problem becomes: 
\begin{equation}
\begin{aligned}\label{opt:cov_est_joint_MM_approxL0}
\underset{\vU,\vXi}{\text{minimize}}&\quad\log\det\left(\vXi\right)+\Tr\left(\vS\vU\vXi^{-1}\vU^T\right)\\
&\quad +\sum_{j=1}^{m}\rho_j \sum_{i=1}^{m}g_p^{\epsilon}\left(u_{ij}\right)\\
\text{subject to}&\quad \vXi\succcurlyeq0,\\
&\quad \xi_{i}\geq \xi_{i+1}, \quad i=1,\dots, q-1,\\
&\quad \xi_{q}\geq \xi_{q+i}, \quad i=1,\dots,m-q,\\
&\quad \vU^T\vU=\vI.
\end{aligned}
\end{equation}
Here $\vU,\vXi\in\bR^{m\times m}$, with $\vXi=\diag(\vxi)\succcurlyeq0$. 

Now, we derive a tight upper bound, $g_{\text{jnt}}\left(\vU,\vXi|\vU^{(k)},\vXi^{(k)}\right)$, for the objective function of \eqref{opt:cov_est_joint_MM_approxL0}, denoted by $f_{\text{jnt}}\left(\vU,\vXi\right)$, at the $\left(k+1\right)$-th iteration.

\begin{table*}[t!]
	\centering
	\begin{tabular}{|l|c|c|c|c|}
		\hline
		& Case 1 & \multicolumn{2}{c|}{Case 2} & Case 3  \\ 
		\hline
		\multirow{4}{*}{Conditions} &  & \multicolumn{2}{c|}{$\alpha_{j-1}^{(k)}< \alpha_{j}^{(k)}\quad\text{if }j>1$} & $\alpha_{q-r-1}^{(k)}< \alpha_{q-r}^{(k)}$\\
		& $\alpha_i^{(k)}\!\!\leq\!\! \alpha_{i+1}^{(k)},\!\!\!\!\! \quad i\!\in\![1\!:\!q\!-\!1]$ & \multicolumn{2}{c|}{$\alpha_i^{(k)}\geq \alpha_{i+1}^{(k)}, \quad i\in[j:j+k-1]$}  & $\alpha_{i}^{(k)}\geq \alpha_{i+1}^{(k)}, \quad i\in[q-r:q-1]$  \\ \cline{3-4}
		& $\alpha_q^{(k)}\!\!\leq\!\! \alpha_{q+i}^{(k)},\!\!\!\!\! \quad i\!\in\![1\!:\!m\!-\!q]$ & $\text{if }j+k<q$ &$\text{if }j+k=q$  & $\alpha_{q}^{(k)}\geq z_{c_i}^{(k)}, \quad i\in[1:k]$\\ 
		&    & $\alpha_{j+k}^{(k)}\!<\!z_{j+k+1}^{(k)}$ & $\!\alpha_q^{(k)}\!\!<\!\! \alpha_{q+i}^{(k)},\!\!\!\!\! \quad i\!\in\![1\!:\!m\!-\!q]\!$  &  $\alpha_{q}^{(k)}< \alpha_{q+i}^{(k)}, \quad i\in[1:m-q]\setminus\cC$\\
		\hline
		\multirow{2}{*}{$\!$Block Updates$\!\!$}     & \multirow{2}{*}{-} & \multicolumn{2}{c|}{\multirow{2}{*}{$\alpha_i^{(k+1)}=\frac{1}{k+1}\sum_{i=0}^{k}\limits\alpha_{j+i}^{(k)},\!\!\! \quad i\in[j:j+k]$}} & $\alpha_i^{(k+1)}=\alpha_i^{(k)},\quad i\!\in\!\cC\setminus\cA$\\
		& & \multicolumn{2}{c|}{} & $\!\!\alpha_i^{(k+1)}\!\!=\!\!\frac{1}{r+p+1}\!\sum_{i=0}^{r}\limits\!\alpha_{q-i}^{(k)}\!+\!\sum_{i=1}^{p}\limits\!\alpha_{a_i}^{(k)},\!\!\quad i\!\in\![q\!-\!r\!:\!q]\bigcup\cA\!\!$\\
		\hline
		Solution  & $\!\!\vphi^\star\!=\!\frac{1+\sqrt{1+4\lambda_{\max}^{(\vS)}\vec{\alpha}^{(k)}}}{2\lambda_{\max}^{(\vS)}}$ & \multicolumn{2}{c|}{-} & -\\
		\hline
	\end{tabular}
	\caption{Updates and optimal solution of the iterative procedure that solves the optimization problem \eqref{opt:phi_joint_MM}.}
	\label{tab:lambda_joint_MM solution}
\end{table*}

\begin{algorithm}[t]
	\caption{JOCE - Joint Optimization for Covariance Estimation for the problem \eqref{opt:cov_est_joint_MM_approxL0}}\label{alg:Alg3}
	\begin{algorithmic}[1]
		\State Set $k=0$, choose $\vU^{(0)}\in\{\vU\vert\vU^T\vU=\vI_q\}$
		\State \textbf{repeat}:
		\State\hspace{\algorithmicindent} Compute $\vphi^{(k+1)}$ from Proposition \ref{prop:lambda_joint_MM}		
		\State\hspace{\algorithmicindent} Compute $\vH_{\text{jnt}}^{(k)}$ with \eqref{eq:H_jnt}
		\State\hspace{\algorithmicindent} Compute $\vV_L$, $\vV_R$, the left and right singular vectors 
		\Statex\hspace{\algorithmicindent} of $-\vH_{\text{jnt}}^{(k)}$, respectively
		\State\hspace{\algorithmicindent} $\vU^{(k+1)} = \vV_L\vV_R^T$
		\State\hspace{\algorithmicindent} $k \gets k+1$
		\State \textbf{until} convergence
		\State Set $\vxi=\frac{1}{\vphi^{(k)}}$
		\State \textbf{return} $\vU^{(k)},\vxi$
	\end{algorithmic}
\end{algorithm}

\begin{prop} The function $f_{\text{jnt}}\left(\vU,\vXi\right)$ is upper bounded by the surrogate function 
	\begin{equation}
		g_{\text{jnt}}\!\left(\!\vU,\vXi|\vU^{(k)},\vXi^{(k)}\right)= g_{\xi}(\vXi) + g_u(\vU) + c_6,
	\end{equation}
	where 
	\begin{equation}
		g_{\xi}(\vXi)\!=\!\log\det\!\left(\vXi\right)\!+\!\Tr\!\left(\!\vG_{\text{jnt}}^{(k)}\vXi\!\right)\! +\! \lambda_{\max}^{(\vS)}\Tr\left(\vXi^{-1}\right),
	\end{equation}
	with
	\begin{equation}\label{eq:G_jnt}
		\vG_{\text{jnt}}^{(k)}=-\left(\vXi^{(k)}\right)^{-1}\!{\vU^{(k)}}^T\!\!\left(\vS-\lambda_{\max}^{(\vS)}\vI_m\right)\vU^{(k)}\!\left(\vXi^{(k)}\right)^{-1}
	\end{equation}
	and 
	\begin{equation}
		g_{u}(\vU)=2Tr\left({\vH_{\text{jnt}}^{(k)}}^T\vU\right),
	\end{equation}
	with
	\begin{equation}\label{eq:H_jnt}
		\vH_{\text{jnt}}^{(k)}=\vH^{(k)}+\left(\vS-\lambda_{\max}^{(\vS)}\vI_m\right)\vU^{(k)}\left(\vXi^{(k)}\right)^{-1}.
	\end{equation}		
	The term $\vH^{(k)}$ is given by \eqref{eq:H_cov} while $c_6$ is an optimization irrelevant constant.
\end{prop}

\begin{proof}
	Based on Lemma \ref{lemma:zero norm upperbound} we can upper bound the third term of the objective with the function $2\Tr\left({\vH^{(k)}}^T\vU\right) + c_4$,
	with $\vH^{(k)}$ given by \eqref{eq:H_cov}.
	
	The second term of the objective function of \eqref{opt:cov_est_joint_MM_approxL0}, denoted by $f$, is jointly convex on $\vU,\vXi=\diag(\vxi)$. One way to establish convexity of $f$ is via its epigraph using the Schur complement:
	\begin{equation*}
	\epi(f)=\!\Bigg\{\!\left(\vU,\vxi,t\right)\Bigg\vert\diag(\vxi)\succ\vzr,\Large{\left[\begin{smallmatrix} \diag(\vxi\otimes\vone_m) & \tilde{\vu}\\
	\tilde{\vu}^T & t
	 \end{smallmatrix}\right]}\normalsize\!\succcurlyeq\vzr\Bigg\},
	\end{equation*}
	where $\tilde{\vu}=\vect\left(\vS^{1/2}\vU\right)$. Without loss of generality we have assumed that all the eigenvalues $\xi_i$ are strictly positive. The last condition is a linear matrix inequality in $(\vU, \vxi, t)$, and therefore $\epi(f)$ is convex.
	
	We can subtract the maximum eigenvalue of the sample covariance matrix $\vS$ and therefore create a jointly concave term. An upper bound to this term is its first order Taylor expansion. It can be shown that 
	\begin{equation}\label{eq:second term jointMM upper bound}
		\begin{aligned}
			\Tr\left(\vS\vU\vXi^{-1}\vU^T\right)\leq& 2\Tr\left({\vF^{(k)}}^T\vU\right)+\Tr\left(\vG_{\text{jnt}}^{(k)}\vXi\right)\\
			&+\lambda_{\max}^{(\vS)}\Tr\left(\vXi^{-1}\right)+c_5
		\end{aligned} 
	\end{equation}
	where 
	$\vF^{(k)}=\left(\vS-\lambda_{\max}^{(\vS)}\vI_m\right)\vU^{(k)}\left(\vXi^{(k)}\right)^{-1}$
	and
	$\vG_{\text{jnt}}^{(k)}=-\left(\vXi^{(k)}\right)^{-1}\!{\vU^{(k)}}^T\!\!\vF^{(k)}$.
	The constant $c_5$ is given by
	$c_5	\!=\! -\Tr\left(\vF^{(k)}{\vU^{(k)}}^T\right)-\Tr\left(\vG_{\text{jnt}}^{(k)}\vXi^{(k)}\right)$.
	
	We observe that now the variables are decoupled. Thus, by combining the upper bounds for the second and the third term we can derive the functions $g_u(\cdot)$ and $g_{\xi}(\cdot)$, with $\vH_{\text{jnt}}^{(k)}=\vH^{(k)}+\vF^{(k)}$
	and 
	$c_6=c_4+c_5$.
\end{proof}

Now, in every MM iteration we need to solve the following optimization problem:
\begin{equation}
	\begin{aligned}
		\underset{\vU,\vXi}{\text{minimize}}&\quad g_{\xi}(\vXi) + g_u(\vU)\\
		\text{subject to}&\quad \vXi\succcurlyeq0,\\
		&\quad \xi_{i}\geq \xi_{i+1}, \quad i=1,\dots, q-1,\\
		&\quad \xi_{q}\geq \xi_{q+i}, \quad i=1,\dots,m-q,\\
		&\quad \vU^T\vU=\vI.
	\end{aligned}
\end{equation}

Since the variables are decoupled we can optimize each one of them separately. The optimization problem for $\vXi$ becomes:
\begin{equation}\label{opt:xi_joint_MM}
	\begin{aligned}
		\underset{\vxi}{\text{minimize}}&\quad \sum_{i=1}^{m}\left(\log\xi_i+\alpha_i\xi_i+\lambda_{\max}^{(\vS)}\frac{1}{\xi_i}\right)\\
		\text{subject to}&\quad \xi_{i}\geq \xi_{i+1}, \quad i=1,\dots, q-1,\\
		&\quad \xi_{q}\geq \xi_{q+i}, \quad i=1,\dots,m-q,
	\end{aligned}
\end{equation}
where $\vec{\alpha}=\Diag\left(\vG_{\text{jnt}}^{(k)}\right)$.

The above problem is not convex. We can make it convex though with the following simple variable transformation: 
\begin{equation}\label{eq:variable transformation}
	\vphi=\frac{1}{\vxi}.
\end{equation}
Now, the problem becomes
\begin{equation}\label{opt:phi_joint_MM}
	\begin{aligned}
		\underset{\vphi}{\text{minimize}}&\quad \sum_{i=1}^{m}\left(-\log\phi_i+\alpha_i\frac{1}{\phi_i}+\lambda_{\max}^{(\vS)}\phi_i\right)\\
		\text{subject to}&\quad \phi_{i}\leq \phi_{i+1}, \quad i=1,\dots, q-1,\\
		&\quad \phi_{q}\leq \phi_{q+i}, \quad i=1,\dots,m-q,
	\end{aligned}
\end{equation}
which is in a convex form.

Similar to the alternating optimization case, the problem \eqref{opt:phi_joint_MM} does not have a closed form solution. Again, we can find an iterative closed form update of the parameter $\vec{\alpha}$ that will provide the optimal solution. 

We start from the corresponding unconstrained problem whose solution is 
\begin{equation}
	\vphi=\frac{1+\sqrt{1+4\lambda_{\max}^{(\vS)}\vec{\alpha}^{(0)}}}{2\lambda_{\max}^{(\vS)}}, 
\end{equation}
where $\vec{\alpha}^{(0)}=\vec{\alpha}$. We can distinguish the same three cases as for problem \eqref{opt:lambda_block_MM}, where the set $\cA$ now is given by
\begin{equation}
	\cA\!=\!\Bigg\{\!c_i\Bigg\vert\alpha_{c_i}^{(k)}\!\leq\!\frac{1}{r\!+\!l\!-\!i\!+\!1}\!\left(\sum_{s=0}^{r}\alpha_{q-s}^{(k)}\!+\!\!\!\sum_{s=0}^{l-i-1}\!\!\alpha_{c_{l-s}}^{(k)}\!\!\right)\!\!\!\Bigg\}.\label{eq:active set_jointMM2}
\end{equation}

\begin{prop}\label{prop:lambda_joint_MM}
	The iterative closed-form update procedure given in Table \ref{tab:lambda_joint_MM solution} converges to the solution of problem \eqref{opt:phi_joint_MM}.
\end{prop}
\begin{proof}
	The proof of Proposition \ref{prop:lambda_joint_MM} follows the same steps as the proof of Proposition \ref{prop:lambda_block_MM}, thus it is omitted.	
\end{proof}

Having obtained the optimal $\vphi^\star$, it is easy to retrieve $\vxi^\star$ from \eqref{eq:variable transformation}. 

The optimization problem for $\vU$ is the following:
\begin{equation}
	\begin{aligned}\label{opt:procrustes_trace form_jointMM}
		\underset{\vU}{\text{minimize}}&\quad \Tr\left({\vH^{(k)}_{\text{jnt}}}^T\vU\right)\\
		\text{subject to}&\quad \vU^T\vU=\vI_m.
	\end{aligned}
\end{equation}

\begin{prop}
	The optimal solution of the optimization problem \eqref{opt:procrustes_trace form_jointMM} is $\vU^\star=\vV_L\vV_R^T$, where $\vV_L\in\bR^{m\times m}$ and $\vV_R\in\bR^{m\times m}$ are the left and right singular vectors of the matrix $-\vH_{\text{jnt}}^{(k)}$, respectively.
\end{prop}
\begin{proof}
	The proof comes directly from Lemma \ref{lemma:Procrustes}.
\end{proof}

In Algorithm \ref{alg:Alg3} we summarize the above iterative procedure. We will refer to it as JOCE.

\section{NUMERICAL EXPERIMENTS}
\label{sec:num_exper}

\subsection{Random Data Drawn from a Sparse PCA Model}
\label{ssec:exper_rand_data}

In the first experiment, we compare the performance of the proposed IMRP algorithm with a benchmark algorithm GPower$_{\ell_0}$ proposed in \cite{Journee10GeneralizedSPCA}. Note that all four GPower algorithms that are proposed in \cite{Journee10GeneralizedSPCA} have very similar performance in terms of chance of recovery and percentage of explained variance. Thus, it is sufficient to consider only one of them.  

\begin{figure}[t]
	\centering
	\includegraphics[width=0.9\columnwidth]{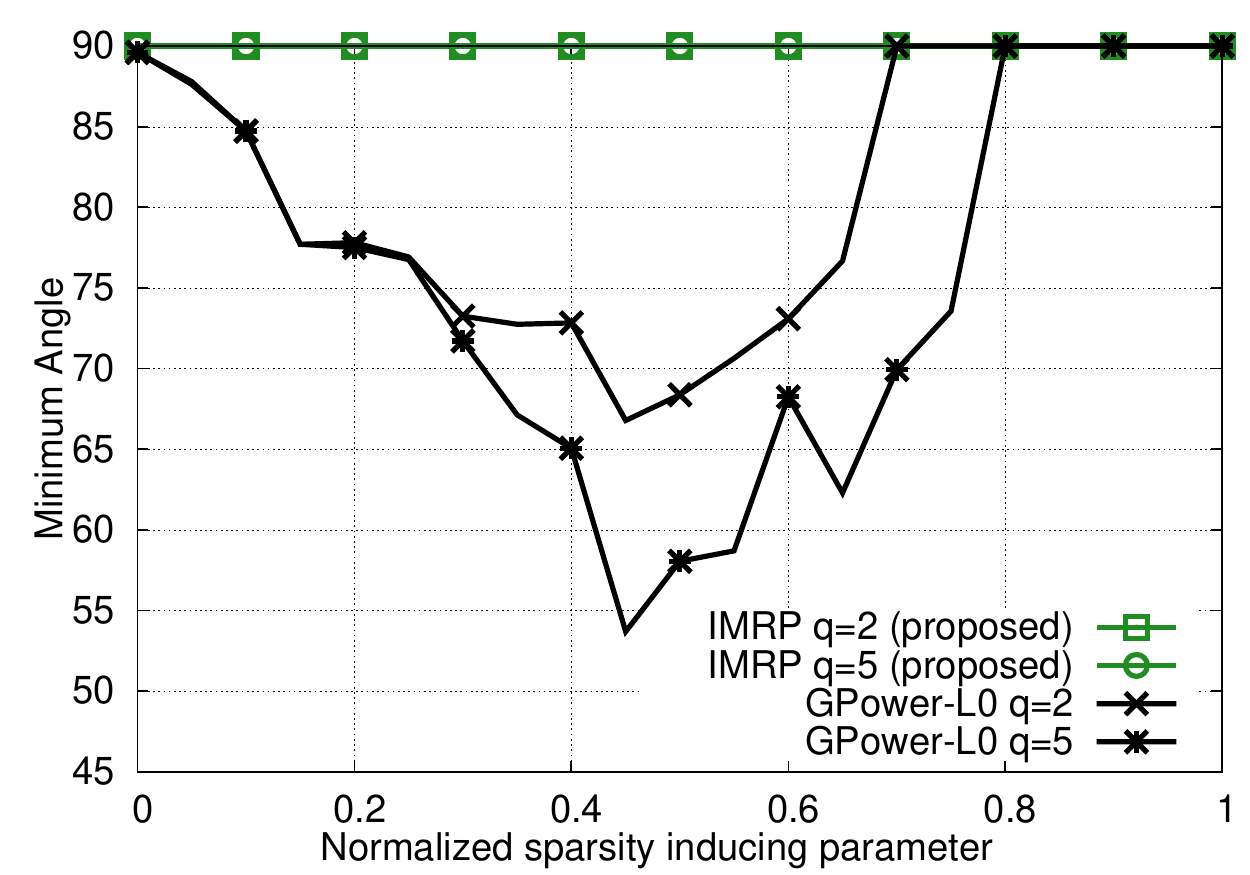}
	\caption{Minimum angle vs normalized regularization parameter.}
	\label{fig:AngMin}
\end{figure}

We first examine the orthogonality of the estimated sparse eigenvectors. We define the angle between eigenvectors $i,j$ as follows:
\begin{equation}
\theta_{ij}\!=\!\min\left(|\arccos\left(\vv_i^T\vv_j\right)|,180^o\!-\!|\arccos\left(\vv_i^T\vv_j\!\right)|\right).
\end{equation}
We consider a setup with $m=500$ and $n=50$. We construct 100 covariance matrices $\vSigma$ through their eigenvalue decomposition $\vSigma=\vV\diag(\vec{\lambda})\vV^T$, where the first $k=5$ columns of $\vV\in\bR^{m\times m}$ are of the following form:
\begin{equation}
	\begin{aligned}
	&\begin{cases}
	v_{ij} \neq 0,&\qquad\text{for $i=1,\dots,10$, $j=1,\dots,5$,}\\
	v_{ij} = 0, &\qquad\text{otherwise},
	\end{cases}
	\end{aligned}
\end{equation}
where the non-zero values are such that the eigenvectors are orthonormal. The remaining eigenvectors are generated randomly, satisfying the orthogonality property. The eigenvalues are set to be $\lambda_i=100(k-i+1)$ for $i=1,\dots,5$, and the rest are set to one. 

For each of the covariance matrix $\vSigma$, we randomly generate 50 data matrices $\vA\in\bR^{m\times n}$ by drawing $n$ samples from a zero-mean normal distribution with covariance matrix $\vSigma$, i.e., i.e., $\va_i\sim\cN(\vzr,\vSigma)$, for $i=1,\dots,n$. Then we employ the two algorithms to compute the first two and the first five sparse eigenvectors. In Figure \ref{fig:AngMin} we plot the minimum angle between any two eigenvectors, i.e., $\min_{i,j}(\theta_{i,j})$ for a wide range of the regularization parameter $\rho$. It is clear that the proposed IMRP algorithm (after thresholding) is orthogonal\footnote{Orthogonality in the sense that $|\vu_i^T\vu_j|\leq\epsilon$, where, in the worst case, $\epsilon$ is in the order of magnitude of the selected threshold $t$, and $i\neq j$. For example, for $t=10^{-12}$, the inner product $|\vu_i^T\vu_j|$ is effectively zero for all practical purposes.} for any choice of $\rho$, while for the GPower$_{\ell_0}$ algorithm the are cases that the estimated eigenvectors have angle less than $55^o$. For large values of $\rho$, GPower$_{\ell_0}$ gives orthogonal results since the sparsity level is high and the estimated eigenvectors do not have overlapping support.

Now, to illustrate the sparse recovering performance of our algorithm we generate synthetic data as in \cite{Journee10GeneralizedSPCA,Song15Sparse,Yuan13truncated}. To this end, we construct a covariance matrix $\vSigma$ through the eigenvalue decomposition $\vSigma=\vV\diag(\vec{\lambda})\vV^T$, where the first $q$ columns of $\vV\in\bR^{m\times m}$ have a pre-specified sparse structure. We consider a setup with $m=500$, $n=50$ and $q=2$. We set the first two orthonormal eigenvectors to be
\begin{equation}\label{eq:predefined eigenvectors}
	\begin{aligned}
		&\begin{cases}
			v_{i1} = \frac{1}{\sqrt{10}},&\qquad\text{for $i=1,\dots,10$,}\\
			v_{i1} = 0, &\qquad\text{otherwise},
		\end{cases}\\
		&\begin{cases}
			v_{i2} = \frac{1}{\sqrt{10}},&\qquad\text{for $i=11,\dots,20$,}\\
			v_{i2} = 0, &\qquad\text{otherwise}.
		\end{cases}
	\end{aligned}
\end{equation}
The remaining eigenvectors are generated randomly, satisfying the orthogonality property. We set the eigenvalues to be $\lambda_1 = 400$, $\lambda_2 = 300$ and $\lambda_i = 1$ for $i=3,\dots,500$.

We randomly generate 500 data matrices $\vA\in\bR^{m\times n}$ by drawing $n$ samples from a zero-mean normal distribution with covariance matrix $\vSigma$, i.e., $\va_i\sim\cN(\vzr,\vSigma)$, for $i=1,\dots,n$. Then, we employ the two algorithms to compute the two leading sparse eigenvectors $\vu_1,\vu_2\in\bR^{500}$. We consider a successful recovery when both quantities $|\vu_1^T\vv_1|$ and $|\vu_2^T\vv_2|$ are greater than $0.99$. 

The chance of successful recovery over a wide range of the regularization parameters $\rho_i$ is plotted in Figure \ref{fig:Recovery}. The horizontal axis shows the normalized
regularization parameters that follows the normalization proposed in  \cite{Song15Sparse}, i.e., $\rho_i =  \max_j \|\va_j\|_2^2$. From the figure we can see that the proposed algorithm IMRP achieves a higher chance of exact recovery for a wide range of the parameters.

\begin{figure}[t]
	\centering
	\includegraphics[width=0.9\columnwidth]{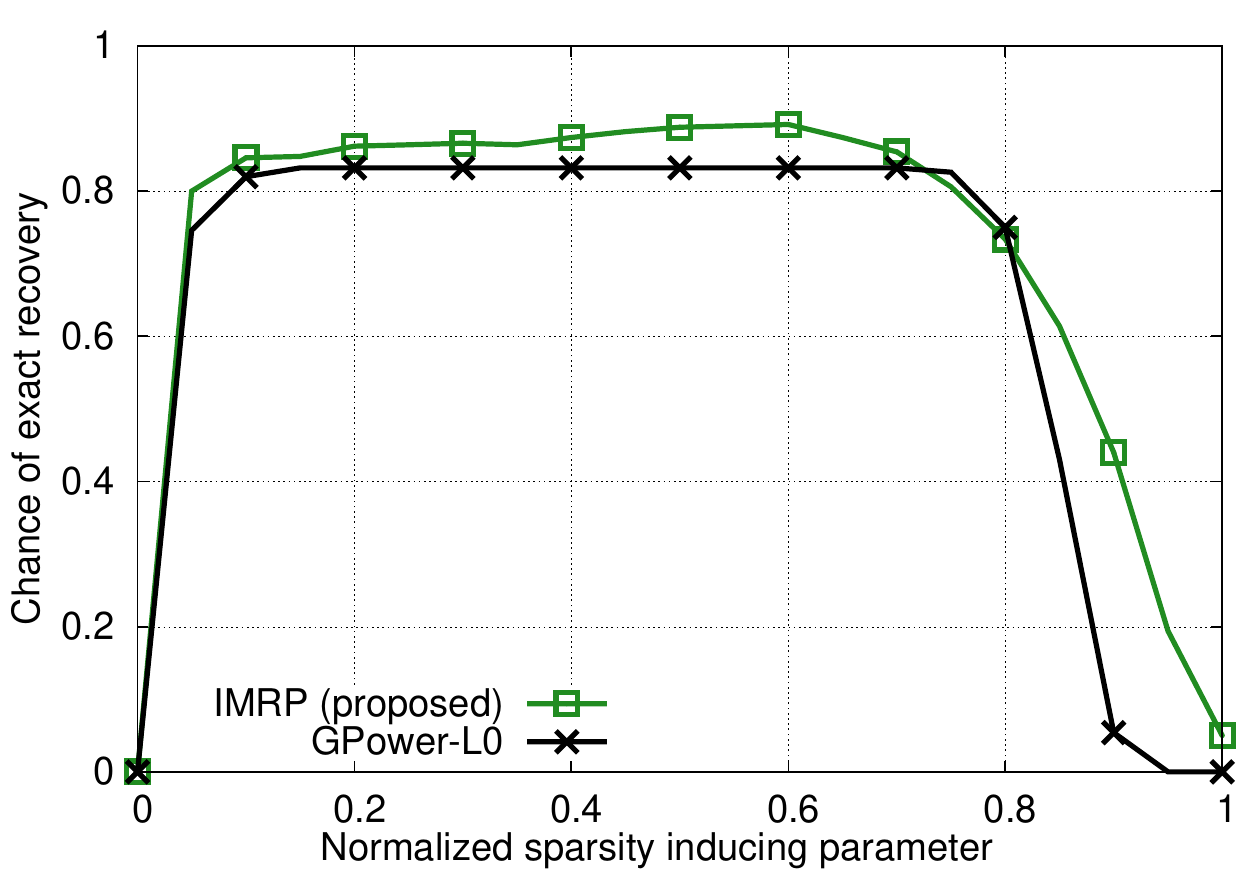}
	\caption{Chance of exact recovery vs normalized regularization parameter.}
	\label{fig:Recovery}
\end{figure}

\subsection{Gene Expression Data}
\label{ssec:exper_gene_data}

In this subsection we compare the performance of the two algorithms on the gene expression dataset collected in the breast cancer study by Bild et al. \cite{Bild06oncogenic}. The dataset contains $158$ samples over $12,625$ genes. We consider the $4,000$ genes with the largest variances and we estimate the  first $5$ eigenvectors.

Notice that due to the orthogonality constraints, increasing the cardinality does not necessarily mean that the CPEV will increase. To this end, for a fixed cardinality, we depict the maximum variance being explained from the sparse eigenvectors up to this cardinality. Thus, the CPEV for cardinality $i$, denoted as CPEV$_i$, is being post-processed as follows:  
\begin{equation}\label{eq:CPEV_i}
\text{CPEV}_i=\max(\text{CPEV}_i,\text{CPEV}_{i-1}).
\end{equation}

{In Figure \ref{fig:ExpVar} we illustrate the cumulative percentage of explained variance, computed by Eq. \eqref{eq:CPEV} and post-processed by \eqref{eq:CPEV_i}, versus the cardinality for the IMRP and GPower$_{\ell_0}$ algorithms. For maximum cardinality the percentage of explained variance becomes $1$ for both algorithms. For fixed cardinality, the two algorithms can explain approximately the same amount of variance. For comparison we have also included the simple thresholding scheme which first computes the regular principal component and then keeps a required number of entries with largest absolute values.

\begin{figure}[t]
	\centering
	\includegraphics[width=0.9\columnwidth]{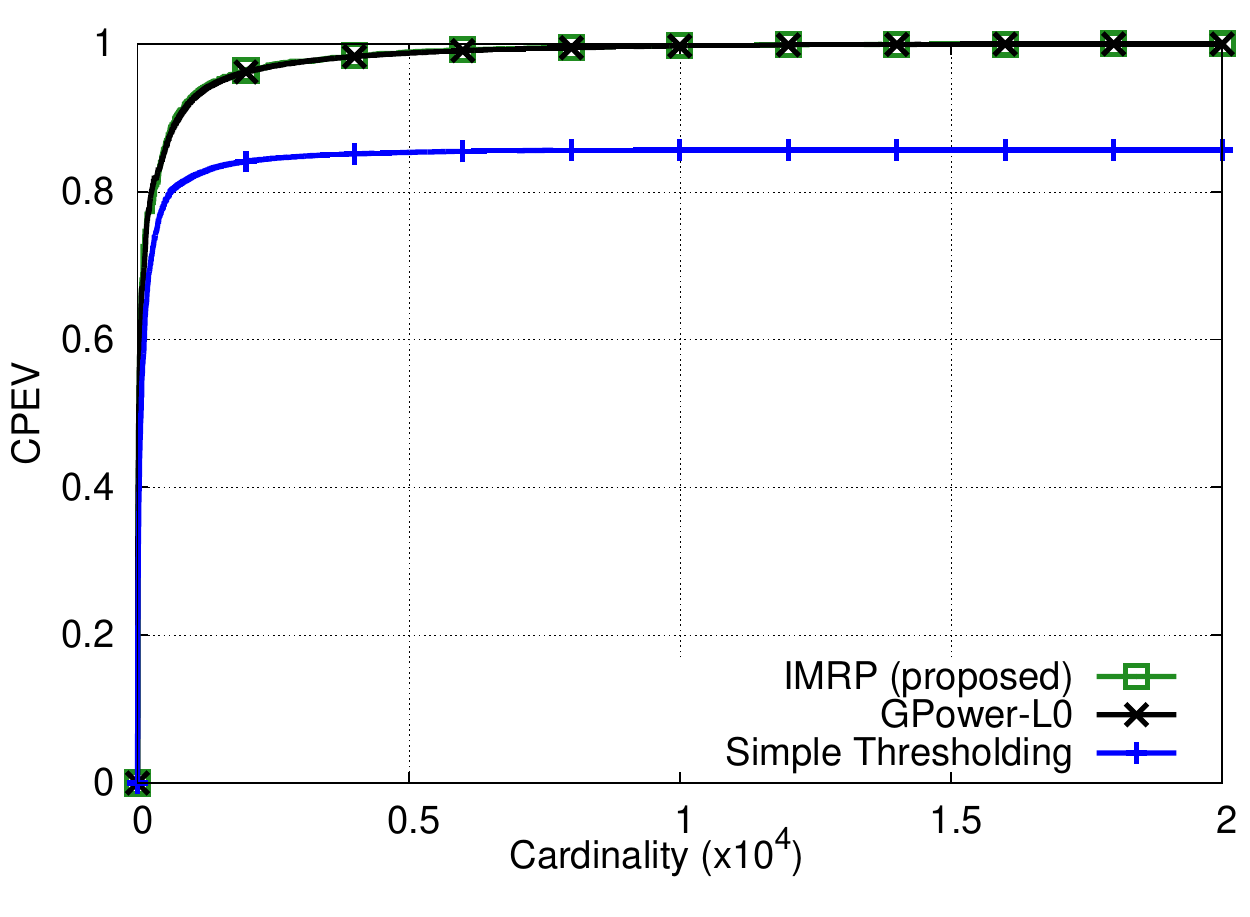}
	\vspace*{-2.5mm}
	\caption{CPEV vs cardinality.}
	\label{fig:ExpVar}
\end{figure}

\subsection{Covariance Estimation}
\label{ssec:covariance_estimation}

In this experiment we consider again the setting and data generation process of section \ref{ssec:exper_rand_data}, with the only difference that we reduce the dimension to $m=200$.
We employ the AOCE and JOCE algorithms to estimate the covariance matrix $\vSigma$. We compute the relative mean square error (RelMSE) for each algorithm, defined as
\begin{equation}
	\text{RelMSE}(\vec{\hat{S}})=1-\frac{\text{MSE}(\vec{\hat{S}})}{\text{MSE}\left(\vS\right)},
\end{equation} 
where $\text{MSE}\left(\vX\right)=\|\vX-\vSigma\|_F^2$,
while $\vec{\hat{S}}$ is the estimated covariance matrix from the two algorithms and $\vS$ is the sample covariance matrix.

From Figure \ref{fig:CovMSE} we observe that AOCE performs better for a small number of samples, while after one point the algorithms have the same performance. Both of the algorithms improve significantly the estimation of the covariance matrix. For example, for $n=m$, the improvement is around $35\%$. For $n\leq m$, instead of $\vS$ we use $\vS_{\text{sh}}$ as defined in \eqref{eq:S_sh}. The parameter $\delta$ is chosen based on a grid search. For this case, in order to show that the improvement in estimation is not due to shrinkage, we include the RMSE for $\vS_{\text{sh}}$. It is clear from the plot that the improvement from shrinkage is around $5\%$. This explains the slight estimation improvement of AOCE and JOCE for $n\leq m$. 

\begin{figure}[t]
	\centering
	\includegraphics[width=0.9\columnwidth]{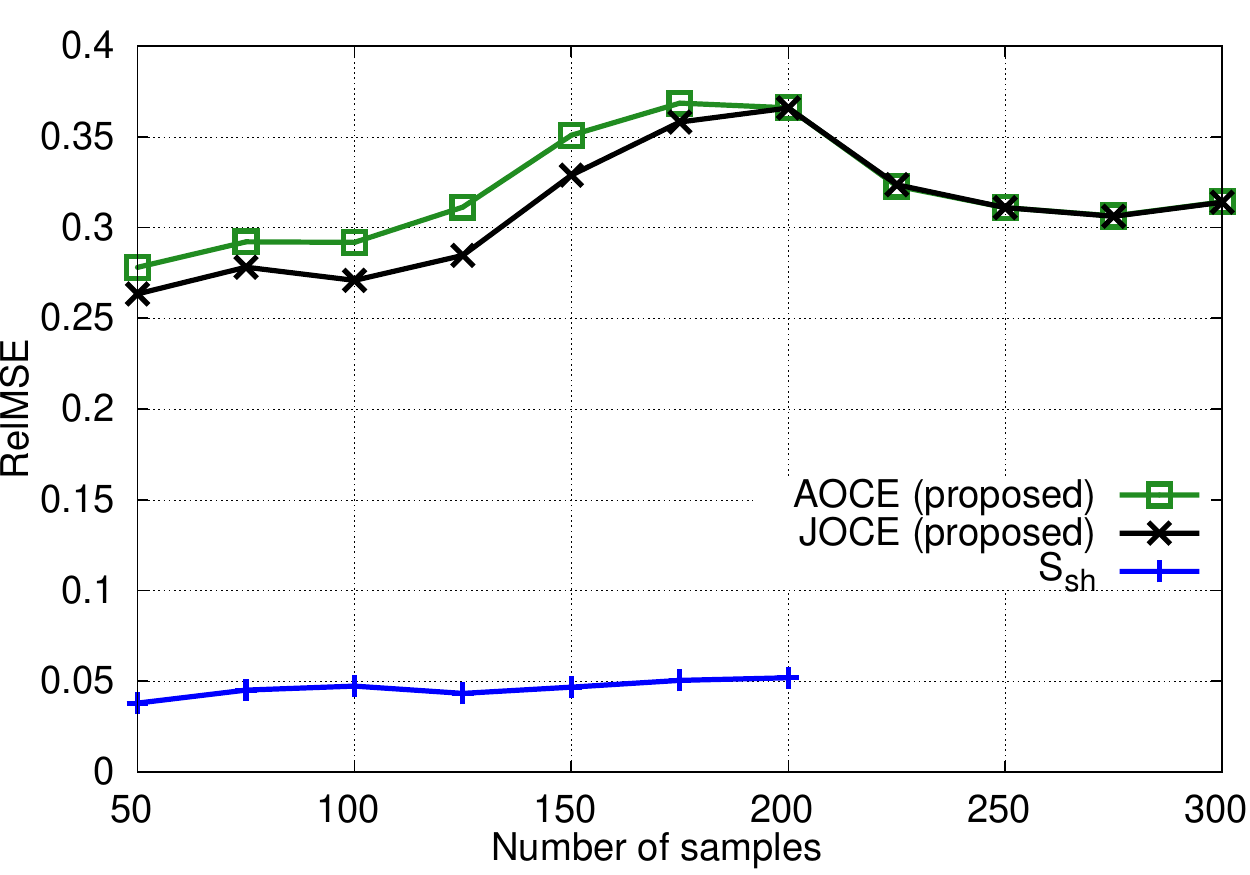}
	\caption{RelMSE vs number of samples.}
	\label{fig:CovMSE}
\end{figure}

\section{CONCLUSION}
\label{sec:conclusion}

In this paper, we first proposed a new algorithm for sparse eigenvalue extraction. The algorithm is derived based on the minorization-majorization method that was applied after a smooth approximation of the $\ell_0$-norm. Unlike all the other state of the art methods, the resulting sparse eigenvectors from our proposed method maintain their orthogonality property. We further formed a covariance estimation problem using the eigenvalue decomposition of the covariance matrix. We simultaneously imposed sparsity on some of the principal eigenvectors to improve the estimation performance. We have proposed two algorithms, based on the MM framework, to efficiently solve the above problem. Numerical experiments have shown that IMRP matches or outperforms existing algorithms while AOCE and JOCE improve significantly the estimation of the covariance matrix.

\appendices
\section{Proof of Lemma \ref{lemma:zero norm upperbound}}
\label{app:proof of zero norm upperbound}

\begin{proof}
Following the same approach as \cite{Song15Sparse}, we can bound the function $\sum_{i=1}^{q}\rho_i \sum_{j=1}^{m}g_p^{\epsilon}\left(u_{ij}\right)$ with a weighted quadratic one. Based on the results of \cite{Song15Sparse} and by incorporating the sparsity parameters $\rho_i$ to the corresponding weights, it holds that
\begin{equation*}
	\sum_{i=1}^{q}\rho_i \sum_{j=1}^{m}g_p^{\epsilon}\left(u_{ij}\right)\leq\vect(\vU)^T\diag\left(\vw\right)\vect(\vU),
\end{equation*}
with the weights $\vw\in\bR_+^{mq}$ given by \eqref{eq:weights}.
Now, the idea is to create a concave term and linearize it since the linear approximation of a concave function is an upper bound of the function. We define $\vw_{\max}\in\bR^q_+$, with $w_{\max,i}$ being the maximum weight that corresponds to the $i$-th eigenvector. For convenience we further define $\vu=\vect(\vU)$, $\vW_d=\diag(\vw)$ and $\vW_m=\diag(\vw_{\max}\otimes\vone_{m})$. Now, we can bound the weighted quadratic function as follows:
\begin{align*}
	\vu^T\vW_d\vu=&\vu^T\left(\vW_d-\vW_m\right)\vu+\vu^T\vW_m\vu\\
	=&\vu^T\left(\vW_d-\vW_m\right)\vu+\vone_m^T\vw_{\max}\\
	\leq& \vu_0^T\!\left(\vW_d-\vW_m\right)\vu_0 \!+\!2\vu_0^T\!\left(\vW_d\!-\!\vW_m\right)(\vu\!-\!\vu_0)\\ 
	&+ \vone_m^T\vw_{\max}\\
	=& 2\vu_0^T\left(\vW_d-\vW_m\right)\vu-\vu_0^T\vW_d\vu_0\\
	&+ 2\left(\vone_m^T\vw_{\max}\right)\\
	=& 2\Tr\left(\vH^T\vU\right) +2\left(\vone_m^T\vw_{\max}\right)\!-\vu_0^T\vW_d\vu_0,
\end{align*}
where $\vH=\big[\left(\vW_d-\vW_m\right)\vu_0\big]_{m\times q}$. This completes the proof.
\end{proof}

\section{Proof of Proposition \ref{prop:lambda_block_MM}}
\label{app:proof of algorithm block MM}

\begin{proof}
	For convenience, in all the proofs we drop the superscript $(k)$ that denotes the current iteration. We denote the updates of $\vz$ by $\bar{\vz}$, i.e., if $\vz=\vz^{(k)}$ then $\bar{\vz}=\vz^{(k+1)}$.
	
	The Lagrangian of the optimization problem \eqref{opt:lambda_block_MM} is
	\begin{align}
	L(\vec{\lambda},\vec{\mu},\vec{\nu})=&-\!\sum_{i=1}^{m}\log\lambda_i\!+\!\sum_{i=1}^{m}z_i\lambda_i\!+\!\sum_{i=1}^{q-1}\mu_i(\lambda_i-\lambda_{i+1}) \nonumber\\ 
	&+ \sum_{i=1}^{m-q}\nu_{q+i}(\lambda_q-\lambda_{q+i}),
	\end{align}
	with $\vlambda\in\bR^m_+$, $\vec{\mu}\in\bR^{q-1}_+$ and $\vec{\nu}\in\bR^{m-q}_+$. Now, we can derive the following Karush-Kuhn-Tucker (KKT) conditions \cite{Boyd04ConvexOptimization}:
	
	\begin{align}
	-\frac{1}{\lambda_1} + z_1 +\mu_1= 0,& \label{KKT partial1}\\
	-\frac{1}{\lambda_i} + z_i +\mu_i-\mu_{i-1}=0,& \quad i=2,\dots,q-1,\label{KKT partial2}\\
	-\frac{1}{\lambda_q}\! +\! z_q\! -\!\mu_{q-1}\! +\!\!\sum_{i=1}^{m-q}\!\nu_i\!=\!0,&\label{KKT partial3}\\
	-\frac{1}{\lambda_{q+i}} + z_{q+i} -\nu_{q+i}=0,& \quad i=1,\dots,m-q,\label{KKT partial4}\\
	\lambda_{i}-\lambda_{i+1}\leq0, &\quad i=1,\dots,q-1,\label{KKT primal1}\\
	\lambda_{q}-\lambda_{q+i}\leq0, &\quad i=1,\dots,m-q,\label{KKT primal2}\\
	\mu_i\geq0,&\quad i=1,\dots,q-1,\label{KKT dual1}\\
	\nu_{q+i}\geq0,&\quad i=1,\dots,m-q,\label{KKT dual2}\\
	\mu_i(\lambda_i - \lambda_{i+1})=0,&\quad i=1,\dots,q-1,\label{KKT compslack1}\\
	\nu_{q+i}(\lambda_q - \lambda_{q+i})=0,&\quad i=1,\dots,m-q\label{KKT compslack2}.
	\end{align}
	
	As a first result we can state the following lemma:
	
	\begin{lemma}
		\label{lemma: blockMM_unconstr_sol}
		The solution of the KKT system \eqref{KKT partial1}-\eqref{KKT compslack2} is $\lambda_i=\frac{1}{z_i}$, for $i=1,\dots,m$, if the following conditions hold:
		\begin{align}
		z_i\geq& z_{i+1}, \quad i=1,\dots,q-1,\label{lemma_unconstr_eq1}\\
		z_q\geq& z_{q+i}, \quad i=1,\dots,m-q. \label{lemma_unconstr_eq2}
		\end{align}
		In this case all the Lagrange multipliers are zero.
	\end{lemma}
	\begin{proof}
		It is straightforward that if inequalities \eqref{lemma_unconstr_eq1} and \eqref{lemma_unconstr_eq2} hold, then the solutions of the primal and dual variables given in the above lemma satisfy all equations. Since the problem is convex, this solution is the optimal.  
	\end{proof}
	
	We can interpret Lemma \ref{lemma: blockMM_unconstr_sol} as follows: if the unconstrained problem has an optimal solution that is inside the feasible region of the constrained problem, then it is also the optimal solution of the constrained problem.
	
	Now, if the conditions of Lemma \ref{lemma: blockMM_unconstr_sol} do not hold, the solution of the unconstrained problem will violate a set of inequality constraints. We can distinguish two different types of violations.
	
	\paragraph{Violations in the first $q$ eigenvalues}
	
	Here, we consider the case where the solution of the unconstrained problem violates the ordering constraints of the first $q$ eigenvalues (Case $2$ of Table \ref{tab:lambda_block_MM solution}). In this case, we need to update the parameters $\vz$ according to the following Lemma:
	\begin{lemma}
		\label{lemma: blockMM_constr_sol_case1}
		For any block of $r$ consecutive inequality violations between the first $q$ eigenvalues, i.e., $\forall j,r$, with $j+r\leq q$, that the following conditions hold
		\begin{align}
		&\quad z_{j-1}> z_{j},\quad \text{if $j>1$},\\
		&\quad z_{i}\leq z_{i+1}, \quad i=j,\dots,j+r-1,\label{eq:viol1}\\
		&\begin{cases}
		z_{j+r}> z_{j+r+1}, \quad &\text{if $j+r<q$},\\
		z_q> z_{q+i},\quad i=1,\dots,m-q,\quad &\text{if $j+r=q$},
		\end{cases}
		\end{align}
		where at least one inequality of \eqref{eq:viol1} is strict, the update of the corresponding block of $\vz$ is
		\begin{equation}\label{eq:z_bar_case1}
		\bar{z}_i=\frac{1}{r+1}\sum_{s=0}^{r}z_{j+s}, \quad i=j,\dots,j+r.
		\end{equation}
		The new KKT system with the updated parameters has the same solution as the original one.
	\end{lemma}
	\begin{proof}
		See Appendix \ref{app:proof of lemma constr case1}.
	\end{proof}
	
	\paragraph{Violations including a set of the last $m-q$ eigenvalues}
	Since we do not impose ordering on the $m-q$ last eigenvalues, any of them could violate the inequality with $\lambda_q$ and not only the neighboring ones. Thus, we use the indices $c_1,\dots,c_k$, with $c_i>q$, for $i=1,\dots,l$, and $c_i\in\cC$, with $\cC$ the set of indices of the eigenvalues that violate the inequality constraints with $\lambda_q$. We further denote by $\cA\subseteq\cC$ the set of indices of the active dual variables $\vec{\nu}$, i.e., $a_i\in\cA$ if $\nu_{a_i}>0$. We assume that $\card(\cA)=p\leq l$. For this type of violations (Case $3$ of Table \ref{tab:lambda_block_MM solution}), the solution is given from the following lemma:
	\begin{lemma}
		\label{lemma: blockMM_constr_sol_case2}
		For any block of $r+l$ consecutive inequality violations between the last $r+1$ ordered and a set of $l$ unordered eigenvalues, i.e., $\forall r,l$, that the following conditions hold
		\begin{align}
		z_{q-r-1}>& z_{q-r},\\
		z_{q-i}\leq& z_{q-i+1}, \quad i=1,\dots,r,\\
		z_{q}\leq& z_{c_{i}},\quad i=1,\dots,l,\label{eq:viol2}\\
		z_{q}>& z_{i},\quad i\in[q+1:m]\setminus\cC,
		\end{align}
		where at least one inequality of \eqref{eq:viol2} is strict, the update of the corresponding block of $\vz$ is
		\begin{equation}\label{eq:z_bar_case2}
		\begin{cases}
		\bar{z}_i\!=\!\frac{1}{r+p+1}\!\left(\sum_{s=0}^{r}\limits z_{q-s}\!+\!\sum_{s=1}^{p}\limits z_{a_s}\!\!\right),\!\!\!\! & i\in[q\!-\!r\!:\!q]\bigcup\cA,\\
		\bar{z}_i=z_i, & i\in\cC\setminus\cA.
		\end{cases}
		\end{equation}
		The set $\cA$ is given by
		\begin{equation}
		\cA\!=\!\Bigg\{c_i\Bigg\vert z_{c_i}\!\geq\!\frac{1}{r\!+\!l\!-\!i\!+\!1}\!\left(\sum_{s=0}^{r}z_{q-s}\!+\!\!\sum_{s=0}^{l-i-1}\!\!z_{c_{l-s}}\!\!\right)\!\!\Bigg\}.\label{eq:active set}
		\end{equation}
		The new KKT system with the updated parameters has the same solution as the original one.
	\end{lemma}
	\begin{proof}
		See Appendix \ref{app:proof of lemma constr case2}.
	\end{proof}
		
	After applying Lemma \ref{lemma: blockMM_constr_sol_case1} and/or \ref{lemma: blockMM_constr_sol_case2}, the new KKT system, apart from equivalent to the original, it further has the exact same form. Thus, we can apply Lemmas \ref{lemma: blockMM_unconstr_sol}-\ref{lemma: blockMM_constr_sol_case2} to the updated system of equations, until we obtain the optimal solution. Since, the original KKT system has $m$ primal and $m-1$ dual variables and in every iteration we effectively remove at least one primal and one dual variable (see Appendix \ref{app:proof of lemma constr case2}), we need at most $m-1$ iterations.	
\end{proof}

\section{Proof of Lemma \ref{lemma: blockMM_constr_sol_case1}}
\label{app:proof of lemma constr case1}

\begin{proof}
	First, we will prove that when an inequality is violated, then the corresponding eigenvalues become equal. Assume that $z_k<z_{k+1}$, with $j\leq k<k+1\leq j+r$. The KKT conditions for this pair are:
	\begin{align}
	-\frac{1}{\lambda_k} + z_k +\mu_k-\mu_{k-1}=0,&\label{eq:KKT_deriv_case1}\\
	-\frac{1}{\lambda_{k+1}} + z_{k+1} +\mu_{k+1}-\mu_{k}=0,&\label{eq:KKT_deriv2_case1}\\
	\lambda_{k}-\lambda_{k+1}\!\leq0,&\label{eq:KKT_primal_case1}\\
	\mu_k\geq0,&\label{eq:KKT_dual_case1}\\
	\mu_k(\lambda_k - \lambda_{k+1})=0.&\label{eq:KKT_cs_case1} 
	\end{align} 
	If we subtract the first two equations we get:
	\begin{equation}
	2\mu_k=z_{k+1}-z_k + \frac{1}{\lambda_k}-\frac{1}{\lambda_{k+1}}+\mu_{k+1}+\mu_{k-1}.
	\end{equation}
	The right hand side of the above equation is strictly positive since $z_{k+1}-z_k>0$, $\frac{1}{\lambda_k}-\frac{1}{\lambda_{k+1}}\geq0$ and $\mu_{k+1},\mu_{k-1}\geq0$. Thus, $\mu_{k}>0$ and from \eqref{eq:KKT_cs_case1} it holds that $\lambda_k=\lambda_{k+1}$. In a similar manner, and using that $\mu_{k}>0$, it is easy to prove that $\mu_{i}>0$, with $i=j,\dots,j+r-1$, which means that $\lambda_j=\dots=\lambda_{j+r}$. 
	
	Having proved the equality of the eigenvalues and that $\vec{\mu}_{[j:j+r-1]}>\vzr$, it is straightforward that the primal feasibility, dual feasibility and complementary slackness are trivially satisfied for this block. Further, the $r+1$ equations of the partial derivative of the Lagrangian reduce to
	\begin{equation}\label{eq:KKT_new_der_case1}
	-\frac{1}{\lambda_i} + \bar{z}_i +\frac{1}{r+1}(\mu_{j+r}-\mu_{j-1})=0,\quad i=j,\dots,j+r,
	\end{equation}
	with $\bar{z}_i$ given by \eqref{eq:z_bar_case1}. We can treat \eqref{eq:KKT_new_der_case1} as only one equation with since it is repeated $r+1$ times. Effectively, we have removed $r$ primal and $r$ dual variables. It is clear that every solution of the reduced set of KKT conditions, is a solution for the original set of KKT conditions.	
\end{proof}

\section{Proof of Lemma \ref{lemma: blockMM_constr_sol_case2}}
\label{app:proof of lemma constr case2}

\begin{proof}
	We write the KKT conditions for the corresponding block in the following form:
	\begin{align}
	-\frac{1}{\lambda_i} +z_i +\mu_i-\mu_{i-1}=0,& \quad i=q-r,\dots,q-1,\label{eq:KKT_deriv1_case2}\\
	-\frac{1}{\lambda_q}\! +\! z_q\! -\!\mu_{q-1}\! +\!\!\sum_{i=1}^{m-q}\!\nu_{q+i}\!=\!0,&\label{eq:KKT_deriv2_case2}\\
	-\frac{1}{\lambda_{a_i}} + z_{a_i} -\nu_{a_i}=0,&\quad a_i\in\cA,\label{eq:KKT_deriv3_case2}\\
	-\frac{1}{\lambda_{d_i}} + z_{d_i} -\nu_{d_i}=0,&\quad d_i\in\cC\setminus\cA,\label{eq:KKT_deriv4_case2}\\
	\lambda_{i}-\lambda_{i+1}\leq0, &\quad i=q-r,\dots,q-1,\label{eq:KKT_primal1_case2}\\
	\lambda_{q}-\lambda_{a_i}\leq0, &\quad a_i\in\cA\label{eq:KKT_primal2_case2}\\	
	\lambda_{q}-\lambda_{d_i}\leq0, &\quad d_i\in\cC\setminus\cA,\label{eq:KKT_primal3_case2}\\	
	\mu_i\geq0,&\quad i=q-r,\dots,q-1,\label{eq:KKT_dual1_case2}\\
	\nu_{q+a_i}\geq0,&\quad a_i\in\cA,\label{eq:KKT_dual2_case2}\\
	\nu_{q+d_i}\geq0,&\quad d_i\in\cC\setminus\cA,\label{eq:KKT_dual3_case2}\\
	\mu_i(\lambda_i - \lambda_{i+1})=0,&\quad i=q\!-\!r,\dots,q\!-\!1,\label{eq:KKT_compslack1_case2}\\
	\nu_{a_i}(\lambda_q - \lambda_{a_i})=0,&\quad a_i\in\cA,\label{eq:KKT_compslack2_case2}\\
	\nu_{d_i}(\lambda_q - \lambda_{d_i})=0,&\quad a_i\in\cC\setminus\cA.\label{eq:KKT_compslack3_case2}
	\end{align}
		
	As in the proof of Lemma \ref{lemma: blockMM_constr_sol_case1}, it is easy to show that $\mu_i>0$, for $i=q-r,\dots,q-1$. This means that $\lambda_{q-r}=\dots=\lambda_q$. Further, assuming that we know the set $\cA$, since $\nu_{a_i}>0$, from complementary slackness we get that $\lambda_q=\lambda_{a_i}$, $\forall a_i\in\cA$. 
	
	Again, having proved the equality of the eigenvalues, and that $\vec{\mu}_{[q-r:q-1]}>\vzr$, $\nu_{a_i}>0$ for $a_i\in\cA$, it is straightforward that equations \eqref{eq:KKT_primal1_case2},\eqref{eq:KKT_primal2_case2},\eqref{eq:KKT_dual1_case2},\eqref{eq:KKT_dual2_case2},\eqref{eq:KKT_compslack1_case2} and \eqref{eq:KKT_compslack2_case2} are trivially satisfied.
	
	The equations \eqref{eq:KKT_deriv1_case2}-\eqref{eq:KKT_deriv4_case2} reduce to
	\begin{align}\label{eq:KKT_new_der_case2}
	-\frac{1}{\lambda_i} +\bar{z}_i + \frac{1}{r+p+1}\left(\sum_{d_i\in\cC\setminus\cA}\nu_{d_i}-\mu_{q-r-1}\right)=0,
	\end{align}
	for $i\in[q-r:q]\cup\cA$ and
	\begin{align}
	-\frac{1}{\lambda_{d_i}} + z_{d_i} -\nu_{d_i}=0,
	\end{align}
	for $d_i\in\cC\setminus\cA$, where $\bar{z}_i$ is given by \eqref{eq:z_bar_case2}. Assuming that $\card(\cA)=p$, we can treat \eqref{eq:KKT_new_der_case2} as only one equation with since it is repeated $r+p+1$ times. Effectively, we have removed $r+p$ primal and $r+p$ dual variables. It is clear that every solution of the reduced set of KKT conditions, is a solution for the original set of KKT conditions.
	
	Now, we will prove that the indices of the active dual variables $\vec{\nu}$ for this iteration are given by \eqref{eq:active set}. 
	
	We consider the case where $z_q\leq z_{c_1}\leq\dots\leq z_{c_l}$, where at least one inequality is strict. We assume that we know the active set of this and any further iteration. First, we will prove by contradiction that $c_l\in\cA$.
	
	Assume that $c_l\notin\cA$. Since $\bar{z}_q$ will be the average of $z_i$'s that are less or equal to $z_{c_l}$, with at least one $z_i$ strictly smaller, it holds that $\bar{z}_q<z_{c_l}$. Now, by adding \eqref{eq:KKT_deriv2_case2} and \eqref{eq:KKT_deriv3_case2}, and subtracting the partial derivative of the Lagrangian corresponding to $c_l$, we get:
	\begin{equation}
	-\frac{1}{\lambda_q}+\frac{1}{\lambda_{c_l}}+\bar{z}_q-z_{c_l}-\mu_{q-r-1}=0.
	\end{equation} 
	The last equation implies that $\mu_{q-r-1}<0$ should hold which is not valid. Thus, $c_l\in\cA$ holds. 
	
	Having proved that $c_l\in\cA$ and that $\vec{\mu}_{[q-r:q-1]}>\vzr$, if the average of $z_{[q-r:q-1]}$ and $z_{c_l}$ is less or equal to $z_{c_{l-1}}$, 
	following the same arguments we can show that $c_{l-1}\in\cA$. Generalizing this result, $c_i\in\cA$ if the following condition is true:
	\begin{equation}
	z_{c_i}\!\geq\!\frac{1}{r\!+\!l\!-\!i\!+\!1}\!\left(\sum_{s=0}^{r}z_{q-s}\!+\!\!\sum_{s=0}^{l-i-1}\!\!z_{c_{l-s}}\!\!\right)
	\end{equation} 
	
	Assuming that $\card(\cA)=p$, the above results states that only the $p$ largest indices of $\cC$ will belong in the active set $\cA$, i.e., $c_i\in\cA$, for $i=l-p+1,\dots,l$. Thus, in order to find the active set, we need to find all the indices that $z_{c_{i}}\geq\bar{z}_q$ is true, where $\bar{z}_q$ is the average of $z_{[q-r:q]}$ and $z_{[c_{i+1}:c_k]}$, as given in \eqref{eq:z_bar_case2}.
\end{proof}

\end{document}